\documentclass[sigconf]{aamas}

\usepackage{balance} 

\usepackage{amssymb}            
\usepackage{mathtools}          
\mathtoolsset{showonlyrefs}     
\usepackage{graphicx}           
\usepackage{xcolor}
\usepackage{subcaption}         
\usepackage[space]{grffile}     
\usepackage{url}                
\usepackage{dsfont}
\usepackage{algorithm}
\usepackage{algpseudocode}
\usepackage{amsmath}
\usepackage{amsthm}
\usepackage{enumitem}
\def\digits#1{%
  \number#1}

\newtheorem*{remark}{Remark}

\theoremstyle{plain}
\newtheorem{theorem}{Theorem}[section]

\newtheorem{lemma}[theorem]{Lemma}

\theoremstyle{definition}
\newtheorem{definition}[theorem]{Definition}

\theoremstyle{remark}

\newcommand{\learnedR}{\ensuremath{\widehat{R}}}
\newcommand{\learnedGamma}{\ensuremath{\widehat{\gamma}}}
\newcommand{\gtR}{\ensuremath{R_0}}
\newcommand{\gtGamma}{\ensuremath{\gamma_{0}}}
\newcommand{\cR}{\ensuremath{R}}
\newcommand{\cGamma}{\ensuremath{\gamma}}

\newcommand{\optimalGamma}{\ensuremath{\widehat{\gamma}^*}}
\newcommand{\optimalPolicy}[2]{\ensuremath{\pi^*_{#1, #2}}}

\newcommand{\expertPolicy}{\ensuremath{\pi^E}}
\newcommand{\approxExpertPolicy}{\ensuremath{\hat{\pi}^E}}
\newcommand{\MDP}{\ensuremath{M}}
\newcommand{\stateSpace}{\ensuremath{S}}
\newcommand{\actionSpace}{\ensuremath{A}}
\newcommand{\transition}{\ensuremath{P}}
\newcommand{\expertDemo}{\ensuremath{D}}

\newcommand{\valueFunction}[3]{\ensuremath{V^{#1}_{#2, #3}}}
\newcommand{\policyComplexity}[1]{\ensuremath{\Pi_{#1}}}
\newcommand{\IRL}{\ensuremath{\Re}}
\newcommand{\approxIRL}{\ensuremath{\hat{\Re}}}
\newcommand{\feasibleSet}[1]{\ensuremath{\mathcal{R}_{#1}}}
\newcommand{\QFunction}[3]{\ensuremath{Q^{#1}_{#2, #3}}}
\newcommand{\VFunction}[3]{\ensuremath{V^{#1}_{#2, #3}}}
\newcommand{\expertFilter}[1]{\ensuremath{B^{#1}}}
\newcommand{\expertFilterC}[1]{\ensuremath{\bar{B}^{#1}}}
\newcommand{\potential}{\ensuremath{\phi}}

\newcommand{\partialMDP}{\ensuremath{\mathcal{M}}}

\newcommand{\piR}[1]{\ensuremath{R^{#1}}}
\newcommand{\piRh}{\ensuremath{\hat{R}^{\pi}}}
\newcommand{\piTransition}[1]{\ensuremath{P^{#1}}}
\newcommand{\hR}{\ensuremath{\hat{R}}}
\newcommand{\optimalA}{\ensuremath{a^*_{s}}}
\newcommand{\optimalAS}{\ensuremath{a^*_{s^*}}}
\newcommand{\optimalS}{\ensuremath{s^*}}
\newcommand{\optimalQ}{\ensuremath{Q^*}}

\newcommand{\Pe}{\ensuremath{P_{a_1}}}
\newcommand{\Pa}{\ensuremath{P_{a}}}

\newcommand{\approPe}{\ensuremath{\widehat{P}_{a_1}}}
\newcommand{\approFFunction}[2]{\ensuremath{\widehat{F}_{#1, #2}}}

\newcommand{\optimalGammaCv}{\optimalGamma_{\text{cv}}}
\newcommand{\optimalGammaCvAll}{\optimalGamma_{\text{oracle}}}

\newcommand{\learnedHorizon}{\ensuremath{\widehat{T}}}
\newcommand{\gtHorizon}{\ensuremath{T_{0}}}
\newcommand{\optimalHorizon}{\ensuremath{\widehat{T}^*}}
\newcommand{\optimalHorizonCv}{\optimalHorizon_{\text{cv}}}
\newcommand{\optimalHorizonCvAll}{\optimalHorizon_{\text{oracle}}}

\usepackage{xspace}
\makeatletter
\DeclareRobustCommand\onedot{\futurelet\@let@token\@onedot}
\def\@onedot{\ifx\@let@token.\else.\null\fi\xspace}

\def\ie{\emph{i.e}\onedot}

\makeatother

\makeatletter
\gdef\@copyrightpermission{
  \begin{minipage}{0.2\columnwidth}
   \href{https://creativecommons.org/licenses/by/4.0/}{\includegraphics[width=0.90\textwidth]{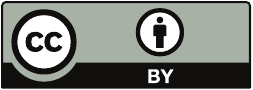}}
  \end{minipage}\hfill
  \begin{minipage}{0.8\columnwidth}
   \href{https://creativecommons.org/licenses/by/4.0/}{This work is licensed under a Creative Commons Attribution International 4.0 License.}
  \end{minipage}
  \vspace{5pt}
}
\makeatother

\setcopyright{ifaamas}
\acmConference[AAMAS '25]{Proc.\@ of the 24th International Conference
on Autonomous Agents and Multiagent Systems (AAMAS 2025)}{May 19 -- 23, 2025}
{Detroit, Michigan, USA}{Y.~Vorobeychik, S.~Das, A.~Nowé  (eds.)}
\copyrightyear{2025}
\acmYear{2025}
\acmDOI{}
\acmPrice{}
\acmISBN{}

\acmSubmissionID{<<OpenReview submission id>>}

\title[AAMAS-2025 Formatting Instructions]{On the Effective Horizon of Inverse Reinforcement Learning}

\author{Yiqing Xu}
\affiliation{
  \institution{National University of Singapore}
  \city{School of Computing}
  \country{Singapore}}
\email{xuyiqing@comp.nus.edu.sg}

\author{Finale Doshi-Velez}
\affiliation{
  \institution{Harvard University}
  \city{Department of Computer Science}
  \country{Cambridge, USA}}
\email{finale@seas.harvard.edu}

\author{David Hsu}
\affiliation{
  \institution{National University of Singapore}
  \city{School of Computing, Smart System Institute}
  \country{Singapore}}
\email{dyhsu@comp.nus.edu.sg}

\begin{abstract}
Inverse reinforcement learning (IRL) algorithms often rely on (forward) reinforcement learning or planning, over a given time horizon, to compute an approximately optimal policy for a hypothesized reward function; they then match this policy with expert demonstrations. The time horizon plays a critical role in determining both the accuracy of reward estimates and the computational efficiency of IRL algorithms. Interestingly, an \emph{effective time horizon} shorter than the ground-truth value often produces better results faster. This work formally analyzes this phenomenon and provides an explanation: the time horizon controls the complexity of an induced policy class and mitigates overfitting with limited data. This analysis provides a guide for the principled choice of the effective horizon for IRL. It also prompts us to re-examine the classic IRL formulation: it is more natural to learn jointly the reward and the effective horizon  rather than the reward alone with a given horizon. To validate our findings, we implement a cross-validation extension and the experimental results support the theoretical analysis. 
The \href{https://effective-horizon-irl.github.io/}{\textcolor{blue}{project page}} 
and \href{https://github.com/eeching/Effective_Horizon_IRL}{\textcolor{blue}{code}} 
are publicly available.
\end{abstract}


\newcommand{\BibTeX}{\rm B\kern-.05em{\sc i\kern-.025em b}\kern-.08em\TeX}

\begin{document}


\pagestyle{fancy}
\fancyhead{}


\maketitle 

\section{Introduction}
\label{sec:introduction}

Inverse reinforcement learning (IRL)  \citep{AndrewNg} aims to infer the underlying task objective from expert demonstrations. One common approach is to estimate a reward function that induces a policy matching closely the demonstrated data. 
This model-based approach holds the promise of generalizing the learned reward function and the associated policy over states not seen in the demonstrations \citep{IM_book}. 

Many existing IRL algorithms follow the classic formulation and assume a known ground-truth discount factor (or equivalently, time horizon) for the expert demonstrations \citep{AndrewNg, Apprenticeship, BIRL, maxentirl, REIRL, GPIRL, DeepMaxEnt, IRL_PG, GCL, GAN-GCL, GAIL, AIRL, fIRL, f_D, IRL_policy_based_learner, pmlr-v139-metelli21a, opirl}. They estimate the reward function based on this specified time horizon. However, we do have the flexibility to choose a different time horizon when learning the reward function and optimizing the policy.

Surprisingly, we find that using a \emph{horizon} shorter than the ground-truth value often produces better results faster, especially when expert data is scarce. Why? Intuitively, the time horizon controls the complexity of an induced policy class. With limited data, a shorter time horizon is preferred, as the induced policy class is simpler and mitigates overfitting. We refer to the horizon or discount factor used during learning as the \emph{effective horizon} or \emph{effective discount factor}, as it holds promise to enhance learning effectiveness.

We present a formal analysis showing that, with limited expert data, using a shorter discount factor or horizon improves the generalization of the learned reward function to unseen states. In IRL settings where the effective horizon varies, the performance gap between the induced policy and the expert policy arises from two sources: (i) reward estimation error due to limited data during IRL, and (ii) policy optimization error from using an effective horizon shorter than the ground-truth. We prove that the effective horizon controls the complexity of the approximated policy class during IRL. As the horizon increases, reward estimation error grows—overfitting occurs because we estimate policies from a more complex class using limited expert data. On the other hand, as the horizon approaches the ground-truth, the policy optimization error decreases. These opposing errors suggest that an intermediate effective horizon balances this trade-off and produces the most expert-like policy.

Based on our theoretical findings, we propose a more natural and higher-performing formulation of the IRL problem: jointly learning the reward function and the effective horizon/discount factor. To validate this approach, we extend the linear programming IRL algorithm (LP-IRL) \citep{AndrewNg} and the maximum entropy IRL algorithm (MaxEnt-IRL) \citep{maxentirl}, using cross-validation to simultaneously learn the reward function and the effective horizon. Our experimental evaluation of these extended algorithms across four different tasks corroborates our theoretical analysis.

 This work presents the first formal analysis of the relationship between the horizon and the performance of the learned reward function in IRL. Previous studies have examined the impact of the horizon in Reinforcement Learning (RL) \citep{guo2022theoretical, hu2022role, laidlaw2023bridging, amit2020discount}, approximate dynamic programming \citep{petrik2008biasing}, and planning \citep{gamma, mattingley2011receding}, but the effects when the reward function is unknown and inferred from data remain under-explored. Our work addresses this gap by providing a theoretical analysis on how changes in the horizon affect policy performance in IRL. 
 Some IRL algorithms employ smaller effective time horizons for computational efficiency \citep{BetweenImitationAndIntention, CostMPC, RHIRL}, while others learn discount factors from data to better align with demonstrations, sometimes incorporating multiple feedback types \citep{giwa2021estimation, ghosal2023effect}. Our theoretical analysis on the role of varying time horizons complements existing work and offers valuable insights to guide future IRL research.
\section{Related works}
\label{sec:related_work}

\paragraph{Effective Horizon of Imitation Learning}
Imitation learning learns desired behaviors by imitating expert data and comprises two classes of methods: model-free behavior cloning (BC) and model-based inverse reinforcement learning (IRL) \citep{IM_book}. The primary distinction between BC and IRL lies in the horizons used to align the learned behaviors with expert data. BC matches step-wise expert actions, resulting in poor generalization to unseen states. 
In contrast, IRL addresses this issue by either matching multi-step trajectory distributions \citep{maxentirl, REIRL, GPIRL, DeepMaxEnt, GCL, GAN-GCL, IRL_PG, IRL_policy_based_learner} or their marginalized approximations \citep{GAIL, AIRL, fIRL, f_D, FAIRL, opirl}. The former employs a double-loop structure to interleave the policy optimization and reward function update, while the latter learns a discriminator to distinguish expert-like behaviors. Both approaches utilize the ground-truth horizon/discount factor for policy optimization, ensuring global temporal consistency between the learned policy and expert. Notably, few IRL methods adopt receding horizons to reduce computational cost\citep{BetweenImitationAndIntention,CostMPC, RHIRL}, claiming shorter optimization horizons yields sub-optimal policies. Moreover, several works focus on finding the optimal discount factors in an IRL context \citep{ghosal2023effect, giwa2021estimation}. However, the absence of a theoretical analysis on the horizon's impact in IRL leaves an important yet overlooked gap in the field.

\paragraph{Theoretical Analysis on Effective Horizon}

Several studies have examined how planning horizons affect Reinforcement Learning (RL)~\citep{guo2022theoretical, hu2022role, laidlaw2023bridging, amit2020discount}, approximate dynamic programming~\citep{petrik2008biasing}, and planning~\citep{mattingley2011receding}, particularly when transition models are estimated from samples~\citep{gamma}. In these works, which assume a known reward function, differences in planning horizons directly affect policy performance. However, the effect of horizons remains unexplored when the reward function is learned from data. Our study addresses this gap by investigating how horizon changes affect policy performance when the reward function is derived from expert demonstrations. Under this setting, policy optimization depends on both the horizon and the horizon-dependent reward function estimation.

This dual dependency introduces unique analytical challenges not present in forward RL or planning. The key challenge lies in examining how varying discount factors influence reward estimates when expert data is limited. While \citet{pmlr-v139-metelli21a} define a set of reward functions that remain compatible with limited expert demonstrations under a fixed, known horizon, we extend this formulation to accommodate unknown and varying horizons. By considering this extended class of feasible reward functions, we analyze how horizon variations affect reward estimation and subsequent policy optimization, particularly with scarce expert data.

\section{Problem formulation}
\label{sec:problem_formulation}

 We consider an MDP $(\stateSpace, \actionSpace, \transition, \gtR, \gtGamma )$, where $\stateSpace$ and $\actionSpace$ represent the state and action spaces, respectively. The transition function is denoted by $\transition: \stateSpace \times \actionSpace \times \stateSpace \to [0, 1]$, and the ground-truth reward function is $\gtR: \stateSpace \times \actionSpace \to [0, R_{max}]$. The discount factor, $\gtGamma$, implicitly determines the value of future rewards at the current time step. The optimal policy, $\optimalPolicy{\gtR}{\gtGamma}$, maximizes the total discounted reward based on $\gtR$ and $\gtGamma$. In our setting, we are given the MDP without the reward function $\gtR$ or the discount factor $\gtGamma$.  
Instead, we have a set of $N$ expert demonstrated state-action pairs $\expertDemo = \{(s_0,a_0), (s_1, a_1), ... (s_{N-1}, a_{N-1})\}$ sampled from a deterministic policy $\optimalPolicy{\gtR}{\gtGamma}$. Our analysis employs the discount factor for simplicity; however, the findings are equally applicable to the planning horizon since both determine how much future rewards are valued. A smaller discount factor effectively limits the agent's planning horizon by diminishing the importance of distant rewards.

We examine the scenario where both the reward function and discount factor $( \learnedR, \learnedGamma )$ are jointly learned from the limited expert demonstrations. The scarcity of data suggests that $(\learnedR, \learnedGamma)$ is susceptible to approximation errors, which consequently affects the induced optimal policy $\optimalPolicy{\learnedR}{\learnedGamma}$.
The quality of the pair \((\learnedR, \learnedGamma)\) is assessed by comparing the value of its induced policy \(\optimalPolicy{\learnedR}{\learnedGamma}\) with the ground-truth optimal policy \(\optimalPolicy{\gtR}{\gtGamma}\), both evaluated under the true \((\gtR, \gtGamma)\) for fairness. We define the loss as \(\left\lVert \valueFunction{\optimalPolicy{\gtR}{\gtGamma}}{\gtR}{\gtGamma} - \valueFunction{\optimalPolicy{\learnedR}{\learnedGamma}}{\gtR}{\gtGamma} \right\rVert_{\infty}\), where \(\valueFunction{\pi}{R}{\gamma}\) is the value function of \(\pi\) under \((R, \gamma)\). The ``best" policy $\optimalPolicy{\learnedR}{\learnedGamma}$ is the one that minimizes this loss. During learning, each choice of \(\learnedGamma\) induces a different policy class \(\Pi_{\learnedGamma}\) with its own complexity, as defined below.

\begin{definition}[Policy Class and Complexity Measure] \label{def:complexity}
For a given \(\learnedGamma\), the policy class \(\Pi_{\learnedGamma}\) comprises all optimal policies for a fixed state space \(\stateSpace\), action space \(\actionSpace\), and transition function \(\transition\), for any reward function \(\cR \in \mathcal{F}_{\cR}\) satisfying the following mild condition. Formally,
\[
\Pi_{\learnedGamma} 
= 
\bigl\{ 
    \pi 
    \;\bigm|\;
    \exists \cR \in \mathcal{F}_{\cR}
    : 
    \pi \text{ is optimal in } 
    (\stateSpace, \actionSpace, \transition, \cR, \learnedGamma )
\bigr\},
\]
where \(\mathcal{F}_{\cR}\) is the set of reward functions such that, for each state \(s\in S\), there is exactly one action \(a^*(s)\) for which \(\cR(s, a^*(s))\) strictly surpasses \(\cR(s, a)\) for all \(a \neq a^*(s)\). The complexity of this policy class is measured by \(|\Pi_{\learnedGamma}|\), the number of distinct optimal policies.
\end{definition}

For a given discount factor $\learnedGamma$, IRL algorithms learn a reward function $\learnedR$ from limited expert data ($N$). This reward function yields an optimal policy $\optimalPolicy{\learnedR}{\learnedGamma}$ from the policy set $\Pi_{\learnedGamma}$ that closely matches the expert data. The complexity of \(\Pi_{\learnedGamma}\), measured by \(\lvert \Pi_{\learnedGamma}\rvert\), influences the degree of overfitting in this low-data regime, making it a key factor in IRL. The optimal effective discount factor \(\optimalGamma\) is defined as the one whose induced optimal policy minimizes the loss:
\begin{equation}
     \optimalGamma 
     = \underset{0 \leq \learnedGamma \leq \gtGamma}{\arg\min}
       \Bigl\lVert 
         \valueFunction{\optimalPolicy{\gtR}{\gtGamma}}{\gtR}{\gtGamma} 
         - \valueFunction{\optimalPolicy{\learnedR}{\learnedGamma}}{\gtR}{\gtGamma}
       \Bigr\rVert_{\infty}.
     \label{equ:optimal_gamma}
\end{equation}

We use this loss to examine the interplay among the amount of expert data, the discount factor, and the resulting policy’s performance. In particular, we investigate how to choose \(\learnedGamma \leq \gtGamma\) to minimize this loss.
\section{Analysis}
\label{sec:main_result}

\subsection{Overview}
We define the \emph{effective horizon} as the horizon (or discount factor) used in learning the reward function and optimizing policies, which may differ from the ground-truth horizon. We formally analyze how this effective horizon influences the quality of the learned reward function under varying amounts of expert data. Our main result, Theorem \ref{theorem:final_bound}, shows that when expert data is limited, using a discount factor smaller than the ground-truth value enables IRL methods to learn reward functions that induce policies more closely aligned with the expert.

\begin{theorem}\label{theorem:final_bound}
Let $(\stateSpace, \actionSpace, \transition)$ be a controlled Markov process shared by two MDPs: the ground-truth MDP $(\stateSpace, \actionSpace, \transition, \gtR, \gtGamma)$ with reward function $\gtR: \stateSpace \times \actionSpace \rightarrow [0, R_{\max}]$ and discount factor $\gtGamma \in (0, 1)$; and the estimated MDP $(\stateSpace, \actionSpace, \transition, \learnedR, \learnedGamma)$ with reward function $\learnedR: \stateSpace \times \actionSpace \rightarrow [0, R_{\max}]$ and discount factor $\learnedGamma \in (0, 1)$, estimated from $N$ expert state-action pairs. Let $|\Pi_{\learnedGamma}|$ denote the complexity of the policy class induced by the estimated effective horizon $\learnedGamma$, and suppose the optimal policy $\optimalPolicy{\learnedR}{\learnedGamma}$ derived from $(\learnedR, \learnedGamma)$ belongs to this class, i.e., $\optimalPolicy{\learnedR}{\learnedGamma} \in \Pi_{\learnedGamma}$. Then, for the optimal policies $\optimalPolicy{\gtR}{\gtGamma}$ and $\optimalPolicy{\learnedR}{\learnedGamma}$ induced by the ground-truth and estimated parameters, respectively, the difference between their value functions evaluated under $\gtR$ and $\gtGamma$ is bounded with probability at least $1 - \delta$ by
\begin{align}   \left\lVert\VFunction{\optimalPolicy{\gtR}{ \gtGamma}}{\gtR}{\gtGamma} - \VFunction{\optimalPolicy{\learnedR}{ \learnedGamma}}{\gtR}{\gtGamma}\right\rVert_{\infty}  \leq   & \frac{2R_{\max}}{(1-\learnedGamma)^2}\sqrt{\frac{1}{2N}\log \frac{|\stateSpace||\Pi_{\learnedGamma}|}{2\delta}} \nonumber \\
& + \frac{\gtGamma -  \learnedGamma}{(1-\gtGamma)(1-\learnedGamma)}R_{max}.
     \label{equ:main-result-bound}
\end{align}
\end{theorem}

Intuitively, Theorem \ref{theorem:final_bound} bounds the performance disparity between the policy induced by the learned $(\learnedR, \learnedGamma)$ and the expert policy as a sum of two terms: the first term bounds the \emph{reward function estimation error} that arises from using a limited expert data and an effective $\learnedGamma$ during IRL, while the second term measures the \emph{policy performance gap} when optimized using $\learnedGamma < \gtGamma$ using that estimated reward function. When $\learnedGamma$ increases, the first term, \ie, the \emph{reward function estimation error}, grows due to overfitting arising from estimating a policy from an increasingly complex class $\Pi_{\learnedGamma}$ using only limited expert data, while the second term, \ie, \emph{the policy optimization error}, diminishes to encourage fidelity to the ground-truth $\gtGamma$ and approaches $0$ when $\learnedGamma \to \gtGamma$. Consequently, these opposing error terms imply that an intermediate value of $\learnedGamma$ yields a better reward function that induces the most expert-like policy.

We build up intermediate theorems and lemmas to formally prove Theorem \ref{theorem:final_bound} in the remaining subsections.  The overall strategy is the following: The overall bound in Theorem~\ref{theorem:final_bound} measures the performance gap between the policy induced by the learned $(\learnedR, \learnedGamma)$ pair and the optimal policy. This gap arises from two sources: (i) differences in the reward functions, and (ii) differences in the horizons over which the policies are optimized. Therefore, we decompose the overall bound into two error terms: (i) the \emph{reward function estimation error} resulting from limited expert data during IRL, and (ii) the \emph{policy optimization error} due to using an effective horizon different from the ground truth. 

After deriving bounds for both error terms, we combine them to prove Theorem~\ref{theorem:final_bound} in Section~\ref{sec:value_error}. The second term is straightforward to bound (see Section~\ref{sec:value_error}), The main challenge lies in bounding the reward function estimation error in the first term (Sections~\ref{sec:complexity_measure} to~\ref{sec:expert_policy_estimation_error}). 
We summarize our strategy on bounding the first error term—reward function estimation error—in terms of the effective horizon \(\learnedGamma\) and the number of expert state-action pairs \(N\) below. 

IRL methods learn the reward function by minimizing the discrepancy between the induced policy and the expert data. Therefore, we bound the reward function estimation error using the \emph{expert policy estimation error}, which measures the gap between the policy induced by the learned reward function and the ground-truth expert policy (see Sections~\ref{sec:feasible_set} and~\ref{sec:reward_error}). To do this, we first establish a correspondence between the feasible reward function set and an expert policy (Lemma~\ref{lemma:feasible_reward} in Section~\ref{sec:feasible_set}). Then, we use this correspondence to bound the error in reward function estimation by the expert policy estimation error (Theorem~\ref{theorem:reward_expert} in Section~\ref{sec:reward_error}). Next, we show that, this expert policy estimation error, in turn, depends on the complexity of the policy class \(|\Pi_{\learnedGamma}|\) and the number of expert pairs \(N\) used to fit it (see Section~\ref{sec:expert_policy_estimation_error}). Furthermore, we prove that the policy class complexity \(|\Pi_{\learnedGamma}|\) is controlled by the effective horizon \(\learnedGamma\): as \(\learnedGamma\) increases, the complexity of the policy class \(\Pi_{\learnedGamma}\) rises (Theorem~\ref{theorem:complexity_measure} in Section~\ref{sec:complexity_measure}). By combining these results, we derive the bound for the reward function estimation error in terms of \(\learnedGamma\) and \(N\) in Section~\ref{sec:expert_policy_estimation_error}. This completes the proof.

In the remainder of this section, we provide the detailed proofs, following the structure outlined above.

\subsection{Feasible Reward Function Set}
\label{sec:feasible_set}
In this section, we establish a correspondence between the given expert demonstrations and the \emph{feasible reward function set}—the set of all pairs of reward functions and horizons consistent with the expert data. This correspondence is essential for later bounding the reward function estimation error by the expert policy estimation error.

To create an algorithm-agnostic mapping from the fixed set of expert data and the effective horizon to the learned reward functions, we extend Metelli et al.’s \citep{pmlr-v139-metelli21a} definition of feasible reward function sets—originally formulated for known discount factors—to variable and unknown effective discount factors. We begin by implicitly defining the feasible reward function set based on the foundational IRL formulation  \citep{AndrewNg}, which includes all pairs of reward-horizon whose induced policies match the expert data. From this implicit definition, we derive an explicit characterization of the feasible reward function set as a function of the expert policy.

We start by first implicitly defining the feasible reward set based on the IRL definition \cite{AndrewNg}, adapting for the varying discount factors.
\begin{definition}[IRL Problem, adapted to the setting of varying discount factors]
Let $\partialMDP = ( \stateSpace, \actionSpace, \transition )$ be the MDP without the reward function or discount factor. An IRL problem, denoted as $\IRL = (\partialMDP, \expertPolicy)$, consists of the MDP and an expert's policy $\expertPolicy$. A reward $\learnedR \in \mathbb{R}^{\stateSpace\times \actionSpace}$ is feasible for $\Re$ if there exists a $\learnedGamma$ such that $\expertPolicy$ is optimal for the MDP $\partialMDP \cup ( \learnedR, \learnedGamma )$, \ie, $\expertPolicy \in \Pi^*_{\learnedR, \learnedGamma}$. We use $\feasibleSet{\IRL}$ to denote the set of feasible rewards for $\IRL$.
\end{definition}

This IRL formulation implies an implicit correspondence between the expert policy $\expertPolicy$ and a \emph{feasible} $(\learnedR, \learnedGamma)$ pair through their advantage function $A^{\expertPolicy}_{\learnedR, \learnedGamma}(s, a) = \QFunction{\expertPolicy}{\learnedR}{\learnedGamma}(s, a) - \VFunction{\expertPolicy}{\learnedR}{\learnedGamma}(s)$. Specifically, the optimal policy induced by $(\learnedR, \learnedGamma)$ matches $\expertPolicy$ when the following two conditions on the advantage function are met:
\begin{enumerate}
    \item if $\expertPolicy(a|s) > 0$, then $A^{\expertPolicy}_{\learnedR, \learnedGamma}(s, a) = 0$,
    \item if $\expertPolicy(a|s) = 0$, then $A^{\expertPolicy}_{\learnedR, \learnedGamma}(s, a) \leq 0$. 
\end{enumerate}
The first condition ensures the expert's chosen actions have zero advantage, while the second guarantees unchosen actions have non-positive advantages.

Next, we establish an explicit correspondence between the estimated $(\learnedR, \learnedGamma)$ and their compatible expert policy $\expertPolicy$ by enforcing these two conditions on the advantage function $A^{\expertPolicy}_{\learnedR, \learnedGamma}$. To do this, we introduce two operators for any given policy $\pi$:
\begin{enumerate}
    \item the \emph{expert-filter}: $(\expertFilter{\pi} A)(s, a) = A(s, a)\mathds{1}\{ \pi(a|s) > 0\}$, that retains the advantage $A(s, a)$ values for actions taken by the expert policy $\expertPolicy(a|s)$,
    \item the \emph{expert-filter-complement}: $(\expertFilterC{\pi} A)(s, a) = A(s, a)\mathds{1}\{\pi(a|s) = 0\}$, that preserves the advantage values for actions \textbf{not} taken by the expert.
\end{enumerate}
Using the Bellman equation, we express the advantage functions in terms of the estimated $(\learnedR, \learnedGamma)$. We then apply the two filters to the advantage function $A^{\expertPolicy}_{\learnedR, \learnedGamma}$ to enforce the optimality conditions, simplifying the expression in the process. This allows us to derive the explicit expression for the feasible reward set presented in Lemma \ref{lemma:feasible_reward}. The detailed derivation is provided in Appendix \ref{appendix:proof_feasible_reward}.

\begin{lemma}[Feasible Reward Function Set, extended from \citet{pmlr-v139-metelli21a}] \label{lemma:feasible_reward}
Let  $\IRL = (\partialMDP, \expertPolicy)$ be an IRL problem. Let $\learnedR \in \mathbb{R}^{\stateSpace \times \actionSpace}$ and $0 < \learnedGamma <1$, then $\learnedR$ is a feasible reward, \ie, $\learnedR \in \feasibleSet{\IRL}$ if and only if there exists $\zeta \in \mathbb{R}^{\stateSpace \times \actionSpace}_{\geq 0}$ and $V \in \mathbb{R}^{\stateSpace}$ such that:
\begin{equation}
    \learnedR = - \expertFilterC{\expertPolicy}\zeta + 
    (E - \learnedGamma \transition)V,
\end{equation}
whereas $E: \mathbb{R}^{|\stateSpace|} \to  \mathbb{R}^{|\stateSpace| \times |\actionSpace|}$ is an operator that marginalizes the action for each state on a function $f(\cdot)$ such that $ (Ef)(s, a) = f(s)$.
\end{lemma}
The feasible reward function in Lemma \ref{lemma:feasible_reward} comprises two terms: the first term depends on the expert policy $\expertPolicy$ and the second term depends on the underlying MDP's transition function. The first term, $-\expertFilterC{\expertPolicy}\zeta$, is derived using the \emph{expert-filter-complement} on a non-negative function $\zeta$. This ensures that actions taken by the expert (i.e., $\expertPolicy(a|s) > 0$) are assigned a value of zero, while actions not taken (i.e., $\expertPolicy(a|s) = 0$) have non-positive values. The second term represents the policy's temporal effect that relies on the MDP's transition function. This can be viewed as reward shaping through the value function, which preserves the expert policy's optimality. 

We have thus established an explicit correspondence between the feasible $(\learnedR, \learnedGamma)$ and the expert policy $\expertPolicy$ in Lemma \ref{lemma:feasible_reward}. This explicit expression will later be used to bound the difference in reward functions by the expert policy estimation error.

\subsection{Reward Function Estimation Error from Expert Policy Estimation Error}
\label{sec:reward_error} 

In this section, we establish a bound on the reward function estimation error in terms of the \emph{expert policy estimation error}—the discrepancy between the estimated expert policy (from limited data) and the true expert policy. Since IRL methods learn the reward function by matching the induced policy to expert data, establishing this bound is crucial, as it allows us to relate the reward estimation error to the amount of expert data and the induced policy complexity in later analysis.

Let's consider two IRL problems, $\IRL = (\partialMDP, \expertPolicy)$ and $\approxIRL = (\partialMDP, \approxExpertPolicy)$, which differ only in expert policies: $\IRL$ utilizes the ground-truth expert policy, while $\approxIRL$ employs an estimated policy from samples. Since an IRL algorithm aligns its induced policy with the estimated expert policy, its feasible sets will be equivalent to that of the estimated expert policy. Intuitively, inaccuracies in estimating the expert policy $\expertPolicy$ lead to errors in estimating the feasible sets $\feasibleSet{\IRL}$. Our goal is to obtain a reward function $\learnedR$ with a feasible set ``close'' to the ground-truth $\gtR$'s feasible set. Specifically, ``closeness'' is determined by the distance between the nearest reward functions in each set. The estimated $\feasibleSet{\approxIRL}$ is considered close to the exact $\feasibleSet{\IRL}$ if, for every reward $\gtR \in \feasibleSet{\IRL}$, there exists an estimated reward $\learnedR \in \feasibleSet{\approxIRL}$ with a small $|\gtR - \learnedR|$ value. 

Given the form of the feasible reward functions corresponding to an expert policy as derived in Lemma \ref{lemma:feasible_reward}, we express the estimation error $|\gtR - \learnedR|$ as a function of the ground-truth expert policy $\expertPolicy$ and the estimated expert policy $\approxExpertPolicy$ from limited data. The bound on the reward function estimation error is shown below.

\begin{theorem}[Extension of Theorem 3.1 in \citet{pmlr-v139-metelli21a}]\label{theorem:reward_expert}
Let  $\IRL = (\partialMDP, \expertPolicy)$ and $\approxIRL = (\partialMDP, \approxExpertPolicy)$ be two IRL problems. Then for any $\gtR \in \feasibleSet{\IRL}$ such that $ \gtR = -\expertFilterC{\expertPolicy}\zeta + (E - \gtGamma \transition)V$ and $\left\lVert\gtR\right\rVert_{\infty} \leq R_{\max}$ there exist $ \learnedR \in \feasibleSet{\approxIRL}$ and $0 < \learnedGamma < 1$, such that element-wise it holds that:
\begin{equation}
    \left|\gtR - \learnedR\right| \leq \expertFilterC{\expertPolicy}\expertFilter{\approxExpertPolicy}\zeta.
\end{equation}
Furthermore, $\left\lVert\zeta\right\rVert_{\infty} \leq \frac{R_{\max}}{1-\gtGamma}$.
\end{theorem}

The theorem above bounds the reward function estimation error by the discrepancy between the true expert policy $\expertPolicy$ and the estimated expert policy $\approxExpertPolicy$ derived from limited data. Intuitively, it states that there exists a reward function $\learnedR$ in the estimated feasible set $\feasibleSet{\approxIRL}$ whose estimation error is controlled by the expert policy estimation error, under the corresponding $\learnedGamma$. Specifically, the error term on the right-hand side is non-zero only for state-action pairs where $\expertPolicy(a \mid s) = 0$ but $\approxExpertPolicy(a \mid s) > 0$. This means the reward estimation error is zero for state-action pairs observed in the expert demonstrations, while errors arise where the estimated expert policy incorrectly assigns positive probability to actions the expert did not take. We refer readers to Appendix \ref{appendix:proof_reward_error} for the detailed proof.

\subsection{Expert Policy Estimation Error from Limited Data}\label{sec:expert_policy_estimation_error}

In this section, we derive a bound on the expert policy estimation error in terms of the amount of expert data \(N\), the effective horizon \(\learnedGamma\), and the induced policy complexity \(|\Pi_{\learnedGamma}|\). We measure this error using the term \(\expertFilterC{\expertPolicy}\expertFilter{\approxExpertPolicy}\zeta\), which quantifies how much the estimated policy \(\approxExpertPolicy\) deviates from the true expert policy \(\expertPolicy\). 

Our derivation proceeds in three steps. First, we derive the expected value of the expert policy estimation error \(\mathbb{E}[\expertFilterC{\expertPolicy} \expertFilter{\approxExpertPolicy}\zeta]\) under our estimation strategy. Then, we apply McDiarmid's inequality to obtain a probabilistic bound on how likely this error deviates from its expected value. Finally, we derive a uniform bound on the estimation error by substituting appropriate threshold values and simplifying the expression.

\paragraph{Expected Value of the Expert Policy Estimation Error}
To estimate the expert policy $\approxExpertPolicy \in \{0, 1\}^{|S|\times |A|}$ from $N$ independent state-action samples, we aggregate the demonstrated pairs by summing their counts into $\widehat{\Pi}^E_N$. For each state $s$, we estimate $\approxExpertPolicy(s, a)$ by selecting the action $a$ with the highest count:
\begin{equation}
    \hat{a}_s = \arg\max_{a} \widehat{\Pi}^E_N(s, a).    
\end{equation}
If the count of \(\hat{a}_s\) is uniquely the highest, we set \(\approxExpertPolicy(s, a)\) to 1 for \(\hat{a}_s\) and 0 for other actions. If no action has a unique maximum count, we set all entries to 0, indicating uncertainty. As the number of samples \(N\) increases, the estimated policy is more likely to reflect the expert’s true decisions. For further details, see Appendix \ref{subsubsection:expert_policy_estimation_strategy}.

Next, based on the estimation strategy described above, we first compute the expected value for the approximated expert policy $\mathbb{E}[\approxExpertPolicy]$ estimated from $N$ expert samples, and in turn that of the estimation error \(\mathbb{E}[\expertFilterC{\expertPolicy}\expertFilter{\approxExpertPolicy}\zeta]\). 

We first compute the expected value of $\approxExpertPolicy(s, a)$ given $N$ expert state-action pairs. The expected number state \(s\) is sampled across the \(N\) samples is $N \cdot \frac{1}{|S|}$. Moreover, since the expert policy is deterministic, the probability that \(a_s^*\) is the unique maximum at state $s$ is the same as the probability that state \(s\) is sampled at least once. Therefore, the expectation of the estimated policy \(\mathbb{E}[\approxExpertPolicy]\) for each state \(s\) is given by:
\[
\mathbb{E}[\approxExpertPolicy(s, a)] =
\begin{cases} 
1 - \left(1 - \frac{1}{|S|}\right)^N & \text{for } a = a^*_s,\\
0 & \text{otherwise}.
\end{cases}
\]

Next, we compute the expected value of the expert policy estimation error $\mathbb{E}[\expertFilterC{\expertPolicy}\expertFilter{\approxExpertPolicy}\zeta]$. Recall the operator \(\expertFilterC{\expertPolicy}\) retains the values of \(\zeta(s, a)\) for actions \(a\) that are not part of the expert policy. The true expert policy \(\expertPolicy\) deterministically selects action \(a_s^*\) for each state \(s\), so the expert-filter complement will only retain the values for actions \(a \neq a_s^*\). Thus, after applying the expert-filter complement, the expected value becomes:
\[
\mathbb{E}[\expertFilterC{\expertPolicy} \expertFilter{\approxExpertPolicy}\zeta(s, a)] = \mathbb{E}[\expertFilter{\approxExpertPolicy}\zeta(s, a)] \cdot \mathbb{E}[\mathds{1}\{ \expertPolicy(a|s) = 0 \}].
\]

Moreover, the operator \(\expertFilter{\approxExpertPolicy}\) preserves the function values for actions that are chosen by the estimated policy \(\approxExpertPolicy(s, a)\). Thus, the expected value of \(\expertFilter{\approxExpertPolicy}\zeta(s, a)\) is:
\begin{align}
\mathbb{E}[\expertFilter{\approxExpertPolicy}\zeta(s, a)] & = \zeta(s, a) \cdot \mathbb{E}[\mathds{1}\{\approxExpertPolicy(a|s) > 0\}]. 
\end{align}
 
Simplifying the expressions above, we have $\mathbb{E}[\expertFilterC{\expertPolicy}\expertFilter{\approxExpertPolicy}\zeta(s, a)] = 0$, since \(\expertPolicy(a|s) = 0\) for all \(a \neq a_s^*\). Intuitively, an expected expert policy estimation error of 0 implies that, on average, the estimated policy \(\approxExpertPolicy\) correctly matches the expert policy \(\expertPolicy\). This means it does not assign non-zero probabilities to actions the expert would not take, indicating no systematic bias in the estimation process.

\paragraph{Applying McDiarmid's Inequality}
Next, we obtain a probabilistic bound on how likely the expert policy estimation error deviates from its expected value $\mathbb{E}[\expertFilterC{\expertPolicy}\expertFilter{\approxExpertPolicy}\zeta(s, a)]$. We do so by applying McDiarmid’s Inequality to the term \(\expertFilterC{\expertPolicy}\expertFilter{\approxExpertPolicy}\zeta\). Since \(\zeta\) is bounded by \(\frac{R_{\max}}{1-\learnedGamma}\), altering any single state-action sample affects at most one row of the estimated policy \(\approxExpertPolicy(s, :)\), which leads to a bounded change in the value of \(\expertFilterC{\expertPolicy}\expertFilter{\approxExpertPolicy}\zeta\) by at most \(\frac{R_{\max}}{1-\learnedGamma}\). This satisfies the bounded difference condition necessary for McDiarmid’s Inequality.

By applying McDiarmid’s Inequality, we derive that the probability of the expert policy estimation error exceeding a threshold \(t\) is bounded by:
\begin{equation}
\Pr\left(\expertFilterC{\expertPolicy}\expertFilter{\approxExpertPolicy}\zeta(s, a) \geq t\right) \leq \exp\left( \frac{-2Nt^2 (1 - \learnedGamma)^2}{R_{\max}^2} \right).
\end{equation}

This bound quantifies how likely this estimation error exceeds a threshold \(t\). As the number of expert-demonstrated state-action pairs \(N\) increases, the probability of significant deviations from the true policy decreases exponentially.

\paragraph{Uniform Bound on the Expect Policy Estimation Error}To obtain a uniform bound on the expert policy estimation error across all \((s, \pi)\) pairs, we apply the union bound and set the right-hand side of the bound to \(\frac{\delta}{|S| |\Pi_{\learnedGamma}|}\). Solving for \(t\), we derive the threshold that bounds the error with probability at least \(1 - \delta\):
\begin{equation}
    t = \frac{R_{\max}}{1-\learnedGamma} \sqrt{ \frac{1}{2N} \ln\left( \frac{|S| |\Pi_{\learnedGamma}|}{\delta} \right) }.
\end{equation}
Since the reward function estimation error is bounded by the expert policy estimation error, \( \left| \gtR - \learnedR \right| \leq \expertFilterC{\expertPolicy} \expertFilter{\approxExpertPolicy} \zeta \), the error bound threshold \( t \) applies to the reward function estimation error with probability at least \( 1 - \delta \). As the amount of expert data \( N \) increases, the estimation error decreases, reflecting improved accuracy with more data. Moreover, the effective horizon \( \learnedGamma \) plays an important role: smaller values of \( \learnedGamma \) reduce both \( \frac{1}{1 - \learnedGamma} \) and \( |\Pi_{\learnedGamma}| \), tightening the bound and better controlling the estimation error.

\subsection{Policy Class Complexity Increases with $\learnedGamma$}
\label{sec:complexity_measure}

In the previous section, we bounded the reward estimation error in terms of the policy class complexity \(|\Pi_{\learnedGamma}|\) and the effective horizon \(\learnedGamma\). In fact, this policy class complexity also depends on the horizon that induces it. In this section, we examine how \(\learnedGamma\) influences policy complexity and prove that, for the policy class in Definition ~\ref{def:complexity}, this complexity grows monotonically with the discount factor \(\learnedGamma\). Consequently, we could deduce how changes in \(\learnedGamma\) affect the overall reward estimation error. The full proof on the monotonicity is provided in Appendix~\ref{appendix:complexity_measure}.

\begin{theorem} \label{theorem:complexity_measure}
Consider a fixed MDP \(\MDP = (\stateSpace, \actionSpace, \transition, \cdot, \cdot)\) with state space \(\stateSpace\), action space \(\actionSpace\), and transition function \(\transition\). For any reward function \(\cR \in \mathcal{F}_{\cR}\) that satisfies the mild condition in Definition~\ref{def:complexity} \footnote{This reward function form is not overly restrictive, as it merely excludes reward functions that assign equal rewards to multiple actions for any given state. In systems where the true reward \( R \notin \mathcal{F}_R \), we can always construct a policy-invariant \( R' \in \mathcal{F}_R \). We assume this specific form of reward function to ensure that discussions about the policy class remain meaningful, as any policy could be considered optimal when arbitrary reward functions are allowed.}, we have the following claims about its policy class:
\begin{enumerate}
    \item $\forall \learnedGamma, \gamma' \in [0, 1)$, if $\learnedGamma < \gamma'$, then $\policyComplexity{\learnedGamma} \subseteq \policyComplexity{\gamma'}$.
    \item When $\learnedGamma = 0$, $|\policyComplexity{0}|$ = 1.
    \item If $\learnedGamma \to 1$, $|\policyComplexity{\learnedGamma}| \geq \left(|A|-1\right)^{|S|-1}|S|$ under mild conditions.  
\end{enumerate}
\end{theorem}

Intuitively, claim 1 asserts that as the discount factor $\learnedGamma$ grows, the number of potentially optimal policies increases monotonically. Together with claim 2 and 3, Theorem \ref{theorem:complexity_measure} demonstrates that there is a steep increase in policy complexity associated with the increment of $\learnedGamma$. Specifically, when the discount factor is at its lowest ($\learnedGamma=0$), there is only one optimal policy, as the reward function has a unique maximum state-action pair for each state. However, as $\learnedGamma$ increases and approaches the largest value ($\learnedGamma \to 1$), the optimal policy class can eventually encompass nearly all possible policies, with $|\policyComplexity{\learnedGamma}| = (|\actionSpace|-1)^{|\stateSpace|-1}|\stateSpace|$. In essence, $\learnedGamma$ effectively controls the complexity of the policy class.

\subsection{Error Decomposition and Deriving the Overall Bound}
\label{sec:value_error}
In this section, we use the reward function estimation error bound from earlier to establish the final bound on \( \left\lVert \VFunction{\optimalPolicy{\gtR}{\gtGamma}}{\gtR}{\gtGamma} - \VFunction{\optimalPolicy{\learnedR}{\learnedGamma}}{\gtR}{\gtGamma} \right\rVert_{\infty} \) presented in Theorem \ref{theorem:final_bound}. Recall that \(\optimalPolicy{\gtR}{\gtGamma} \in \policyComplexity{\gtGamma}\) is the optimal policy derived from the ground-truth parameters \((\gtR, \gtGamma)\), while \(\optimalPolicy{\learnedR}{\learnedGamma} \in \policyComplexity{\learnedGamma}\) is the optimal policy derived from the learned parameters \((\learnedR, \learnedGamma)\). We aim to bound the difference between their value functions evaluated under the ground-truth \(\gtR\) and \(\gtGamma\).

To simplify the analysis, we decompose the overall error into two manageable terms: the first accounts for the difference in value functions caused by the reward function estimation error, \(\left\lVert \gtR - \learnedR \right\rVert_{\infty}\); the second accounts for the value difference due to optimizing the policy under different horizons (\(\gtGamma\) versus \(\learnedGamma\)).

\begin{theorem}[Value Function Difference Bound] \label{theorem:value_bound}
Let $\partialMDP = (\stateSpace, \actionSpace, \transition)$ be a partial Markov Decision Process shared by two MDPs. Let $\gtR : \stateSpace \times \actionSpace \rightarrow [0, R_{\max}]$ be the ground-truth reward function, with $\gtGamma \in (0, 1)$ as the ground-truth discount factor. Let $\learnedR : \stateSpace \times \actionSpace \rightarrow [0, R_{\max}]$ and $\learnedGamma \in (0, 1)$ be the estimated reward function and discount factor obtained from data. Consider the optimal policies $\optimalPolicy{\gtR}{\gtGamma}$ and $\optimalPolicy{\learnedR}{\learnedGamma}$ induced by $(\gtR, \gtGamma)$ and $(\learnedR, \learnedGamma)$, respectively. Then, the difference between their value functions, evaluated using the ground-truth reward $\gtR$ and discount factor $\gtGamma$, is bounded above by
\begin{align}
    \left\lVert \VFunction{\optimalPolicy{\gtR}{\gtGamma}}{\gtR}{\gtGamma} - \VFunction{\optimalPolicy{\learnedR}{\learnedGamma}}{\gtR}{\gtGamma} \right\rVert_{\infty} \leq \frac{2}{1 - \learnedGamma} \left\lVert \gtR - \learnedR \right\rVert_{\infty} + \frac{|\gtGamma - \learnedGamma|}{(1 - \gtGamma)(1 - \learnedGamma)} R_{\max}.
    \label{equ:theorem2_bound}
\end{align}
\end{theorem}
Proof of Theorem \ref{theorem:value_bound} is in Appendix \ref{appendix:proof_value_error}.

Recall that we have established a bound on the reward estimation error \(\left\lVert \gtR - \learnedR \right\rVert_{\infty}\) in terms of the number of expert data \(N\), the effective horizon \(\learnedGamma\), and the policy complexity \(|\Pi_{\learnedGamma}|\) in Section \ref{sec:expert_policy_estimation_error}. Specifically, with probability at least \(1 - \delta\), the reward function estimation error is bounded by $\frac{R_{\max}}{1-\learnedGamma} \sqrt{ \frac{1}{2N} \ln\left( \frac{|S| |\Pi_{\learnedGamma}|}{\delta} \right) }$. Substituting this bound into the first error term on the right-hand side of Theorem \ref{theorem:value_bound}, we obtain:

\begin{align}
    \left\lVert \VFunction{\optimalPolicy{\gtR}{\gtGamma}}{\gtR}{\gtGamma} - \VFunction{\optimalPolicy{\learnedR}{\learnedGamma}}{\gtR}{\gtGamma} \right\rVert_{\infty} \leq\ & \frac{2R_{\max}}{(1-\learnedGamma)^2} \sqrt{ \frac{1}{2N} \ln\left( \frac{|S| |\Pi_{\learnedGamma}|}{\delta} \right) }  \\
    & + \frac{|\gtGamma - \learnedGamma|}{(1 - \gtGamma)(1 - \learnedGamma)} R_{\max}.
    \label{equ:final_derivation}
\end{align}

This completes the proof of Theorem \ref{theorem:final_bound}. The overall error consists of two terms. The first term arises from the reward function estimation error, which increases with larger \(\learnedGamma\). A larger \(\learnedGamma\) leads to a more complex induced policy class, making overfitting more likely when expert data is limited. Consequently, the reward function estimation error increases as the expert policy estimation error grows significantly. The second term represents the performance loss from using a smaller discount factor (\(\learnedGamma < \gtGamma\)), and decreases as \(\learnedGamma\) approaches \(\gtGamma\). These opposing dependencies on \(\learnedGamma\) suggest that there exists an intermediate value \(0 < \learnedGamma < \gtGamma\) that minimizes the overall loss. We will empirically demonstrate this in the following section.
\section{Experiment}
\label{sec:experiments}
In this section, we empirically examine Theorem \ref{theorem:final_bound} by exploring how discount factors influence IRL performance with varying amounts of expert data. We adapt Linear Programming IRL (LP-IRL) \citep{AndrewNg} and Maximum Entropy IRL (MaxEnt-IRL) \citep{maxentirl} to accommodate different discount factor settings and expert data sizes (implementation details are provided in Appendices \ref{appendix:LP-IRL_objective} and \ref{appendix:maxentirl}).

To jointly learn \((\learnedR, \learnedGamma)\), we incorporate cross-validation into our modified IRL methods to optimize the discount factor (details in Section \ref{sec:cross-val}). By evaluating policies across a range of $\learnedGamma$s, cross-validation allows us to directly observe how varying \(\learnedGamma\) affects IRL performance with different amounts of expert data. While specialized IRL methods exist for optimizing \(\learnedGamma\) \citep{giwa2021estimation, ghosal2023effect}, we choose cross-validation because it not only identifies the optimal discount factor, but also reveals performance trends across various settings, which is crucial for validating the nuanced implications of our theoretical findings.
Specifically, we answer the following questions through cross-validation:
\begin{enumerate}[label={Q.\protect\digits{\theenumi}}]
    \item Can a lower $\learnedGamma < \gtGamma$ improve IRL policy performance?
    \item  How does $\optimalGamma$ change with increasing expert data $N$?
    \item  Is the cross-validation extension effective in finding $\optimalGamma$?
\end{enumerate}
\begin{figure*}
    \centering
    \begin{tabular}{c}
    \includegraphics[width=1\textwidth]{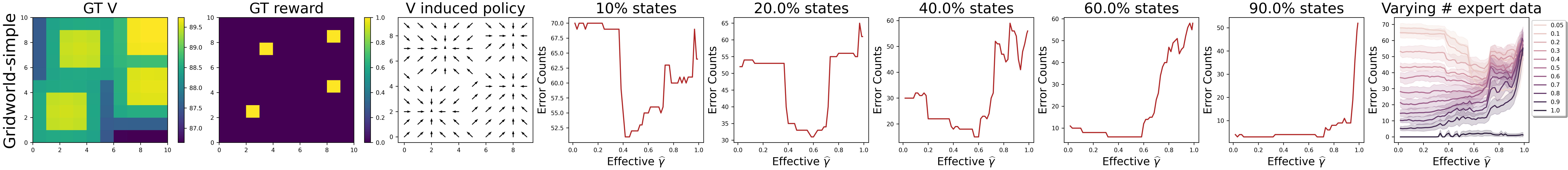}\\
    \includegraphics[width=1\textwidth]{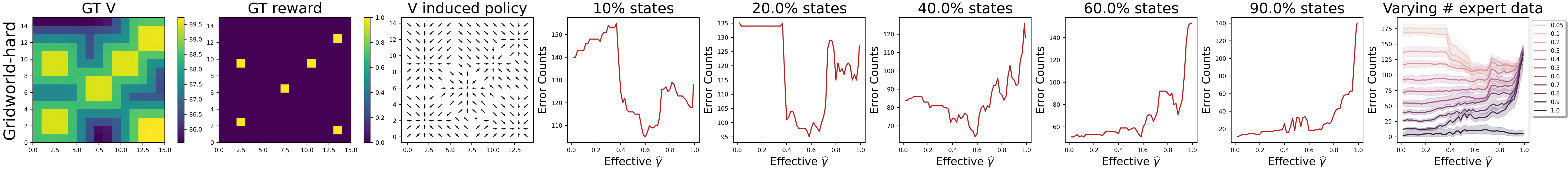}\\
    \includegraphics[width=1\textwidth]{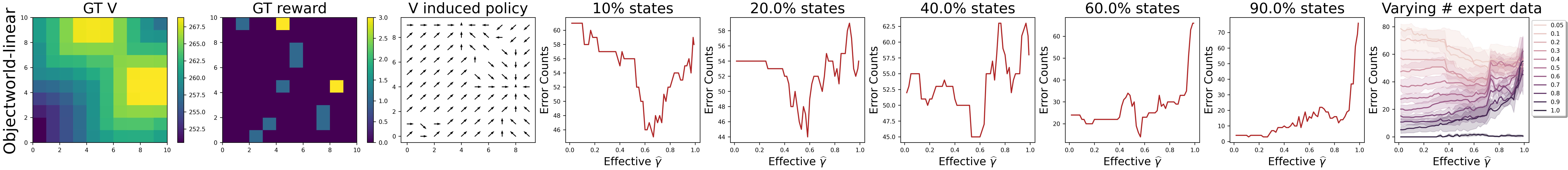}\\ 
    \includegraphics[width=1\textwidth]{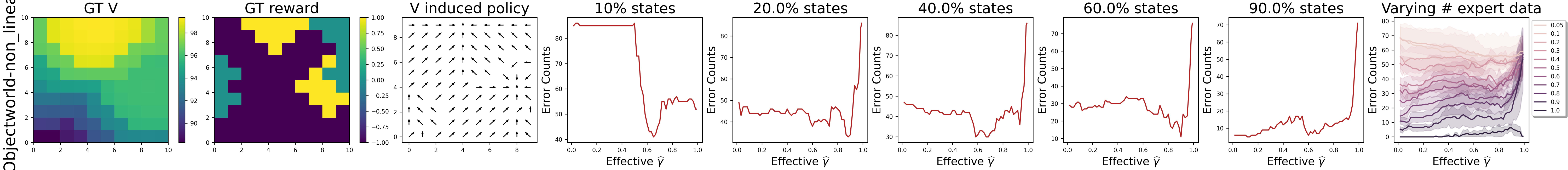}
    \end{tabular}
    \caption{Summary of LP-IRL with varying discount factors across four tasks. The \textit{error counts} measure the number of states for which a policy's action selection deviates from the expert's actions. Each task displays the ground-truth value function (column 1), reward function (column 2), expert policy (column 3), error count curves for different amount of expert data in a \textbf{single instance} (columns 4-8), and the error count curve summary for a \textbf{batch} of 10 MDPs across varying amount of expert data (column 9). In all four tasks, $\gtGamma = 0.99$. The optimal discount factor $\optimalGamma<\gtGamma$ for varying amount of expert data. MaxEnt-IRL has similar curves in Figure \ref{fig:maxent_envs}.} 
    \vspace{-10pt}
    \label{fig:lp_envs}
    \Description{Summary of LP-IRL with varying discount factor}
\end{figure*}

We evaluate the performance of LP-IRL and MaxEnt-IRL on four Gridworld and Objectworld tasks with varying reward complexity. We use the number of incorrectly induced actions as a proxy for value estimation error in the overall error bound. This measure counts the number of states where the induced policy $\approxExpertPolicy$ differs from $\expertPolicy$ in action selection. For Q.1, we measure the number of incorrectly induced actions under varying discounted factors and different amount of expert data. Our findings show that the optimal $\learnedGamma$s across all expert amount are smaller than $\gtGamma$ for both algorithms. For Q.2, we plot how the optimal discount factors change as the number of expert data increases. The consistent U-shaped curves observed in all cases align with the anticipated overfitting effect implied by the second error term in Equation \ref{equ:main-result-bound}. For Q.3, we compare the performance of policies selected via cross-validation with the best policy learnable from the available expert data. Our results indicate that the discrepancy in performance is negligible for all tasks, demonstrating the effectiveness of cross-validation in selecting $\optimalGamma$. 

\subsection{Task Setup}

We design four tasks of varying complexity in reward functions: Gridworld-simple, Gridworld-hard, Objectworld-linear, and Object-world-nonlinear, adapted from \cite{GPIRL} and \cite{AndrewNg}. We illustrate each task instance in the first three columns of Tables~\ref{fig:lp_envs} and more details on the task specification are in Appendix \ref{appendix:task_details}. The ground-truth discount factor is $\gtGamma = 0.99$. All task MDPs are assumed to be ergodic, such that all states are reachable and interconnected, allowing IRL methods to estimate rewards and policies for all states from limited expert data by propagating values from observed states during value iteration. 

The Gridworld tasks provide sparse rewards only at randomly sampled goals: Gridworld-simple has fewer goals ($4$) and a smaller state space ($10\times10$ states), while Gridworld-hard has more goals ($6$) and a larger state space ($15\times15$ states). On the other hand, the Objectworld tasks have denser ground-truth rewards that are functions of nearby object features. The reward function for Objectworld-linear is linear with respect to the features of nearby objects, while that of Objectworld-nonlinear is non-linear. Intuitively, learning a complex reward function may be more susceptible to overfitting, especially when expert data is sparse compared to the state space.

\begin{figure}[t]
\centering
\includegraphics[width=1\columnwidth]{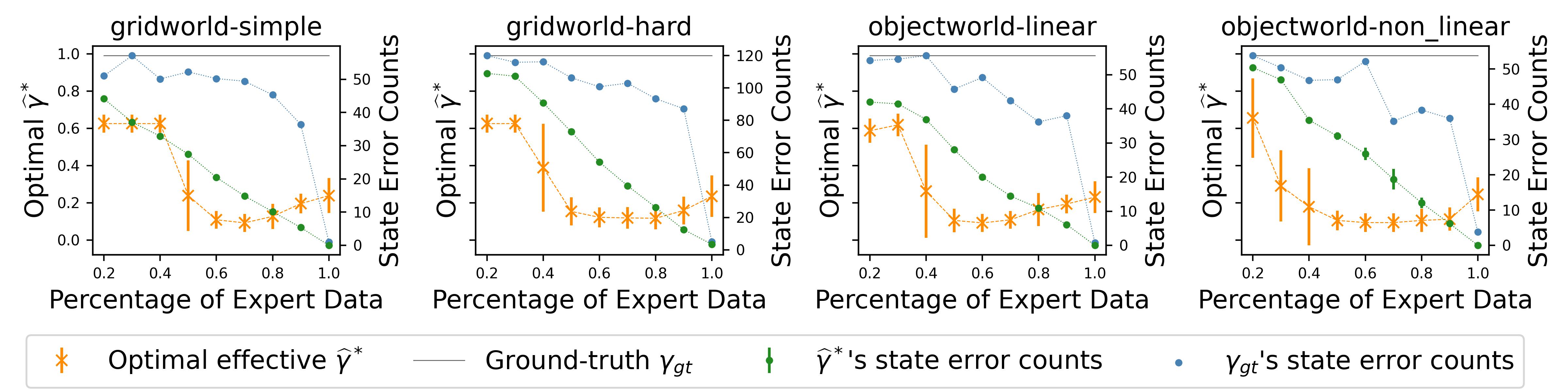}
\caption{Optimal $\optimalGamma$ for LP-IRL at varying amount of expert data. For each task, we select \(\optimalGamma\) for all 10 sampled environments through cross-validation. The orange curves illustrate how the optimal discount factor \(\optimalGamma\) changes with the amount of expert data, while the green curves show the corresponding error counts. The ground-truth \(\gtGamma = 0.99\) is depicted in grey, with its error counts displayed in blue. As the amount of expert data increases, \(\optimalGamma\) initially decreases sharply and then gradually increases, indicating that overfitting is prominent when expert data is scarce.}
\vspace{-15pt}
\label{fig:gamma_vs_coverage}
\Description{Gamma vs Coverage for LP-IRL}
\end{figure}

We vary the amount of expert data by adjusting the number of state-action pairs included, measuring this as a percentage relative to the state space size. Since states may be sampled multiple times in the trajectories, the total number of state-action pairs \( N \) can exceed the number of states \( |S| \). Given expert trajectories \( D = \{ \tau_0, \tau_1, \dots \} \), we define the percentage of expert data as \( K\% = \frac{N}{|S|} \times 100\% \). We evaluate the performance of the induced policy by counting the number of states where it selects a different action from the expert policy—referred to as the \emph{state error count}.

\subsection{Cross Validation Extension}\label{sec:cross-val}

We use cross-validation to determine the optimal \(\optimalGamma\) from the expert data \(D\) containing \(N\) state-action pairs. We split \(D\) into non-overlapping training (80\%) and validation (20\%) sets. We uniformly sample \(M = 20\) discount factors from the interval \((0, \gtGamma)\). For each sampled \(\gamma\), we learn \(R_{\gamma}\) using the training set and evaluate the induced policy on the validation set by counting the number of states where it selects a different action from the expert policy—the \emph{state error count}. The \(\gamma\) that minimizes this error count is selected as the optimal \(\learnedGamma\). For all tasks, we randomly sample 10 environments per task and report the mean and standard deviation of errors. To assess the effectiveness of cross-validation, we employ an oracle representing the best policy learnable from all available expert data. This oracle, considered "cheating," uses both the training and validation sets for training and validates on the entire state space (both observed and unobserved states). We use this oracle only to compare whether the \(\optimalGamma\) chosen by cross-validation corresponds to that of the oracle policy, not for selecting the optimal \(\learnedGamma\).

\subsection{Results}
We assess the impact of the \textit{effective horizon} on IRL by evaluating LP-IRL and MaxEnt-IRL across four tasks. For simplicity, we treat LP-IRL's discount factor \(\gamma\) and MaxEnt-IRL's horizon \(T\) interchangeably, with findings for \(\gamma\) also applying to \(T\) unless specified otherwise. Policy performance results are presented in Tables~\ref{fig:lp_envs} (LP-IRL) and \ref{fig:maxent_envs} (MaxEnt-IRL). We report \emph{state error counts}, measuring discrepancies between induced and expert policies by counting the states where their action selections differ.

\textbf{Q.1 Optimal $\optimalGamma$ is Lower than Ground-Truth.}

As shown in Figures \ref{fig:lp_envs} and \ref{fig:maxent_envs}, the optimal discount factor $\optimalGamma < \gtGamma$  for all four tasks and across various amounts of expert data in both LP-IRL and MaxEnt-IRL. With limited expert data, the error count curves are generally U-shaped: discrepancies with the expert policy decrease as \(\learnedGamma\) increases to a ``sweet spot'' and then rise sharply. This pattern confirms our error bounds in Theorem \ref{theorem:final_bound}: for small \(\learnedGamma\), the overfitting-related error (second term in Equation \ref{equ:main-result-bound}) is less significant, and increasing \(\learnedGamma\) allows temporal extrapolation, reducing the overall error. However, beyond the optimal \(\learnedGamma\), overfitting becomes more pronounced, and the overall error increases as the first error term outweighs the benefits.

With abundant expert data, error counts remain low (in LP-IRL) or initially decrease (in MaxEnt-IRL) for small \(\learnedGamma\) and then increase as \(\learnedGamma\) grows, indicating that \(\optimalGamma < \gtGamma\) yields the most expert-like policy, confirming our theoretical results. Interestingly, in LP-IRL, the error counts do not initially drop as \(\learnedGamma\) increases. This is because, with dense expert data, LP-IRL accurately matches step-wise behaviors, making performance gains from temporal reasoning negligible. This observation supports \citet{FeedbackII}'s insight that naive behavioral cloning can outperform IRL algorithms when ample expert data is available. In contrast, MaxEnt-IRL does not exhibit low error counts for small \(\learnedGamma\) because its reward function is parameterized linearly in state features, limiting its ability to precisely replicate actions even with abundant data.

\begin{figure}[t]
\centering
\includegraphics[width=1\columnwidth]
{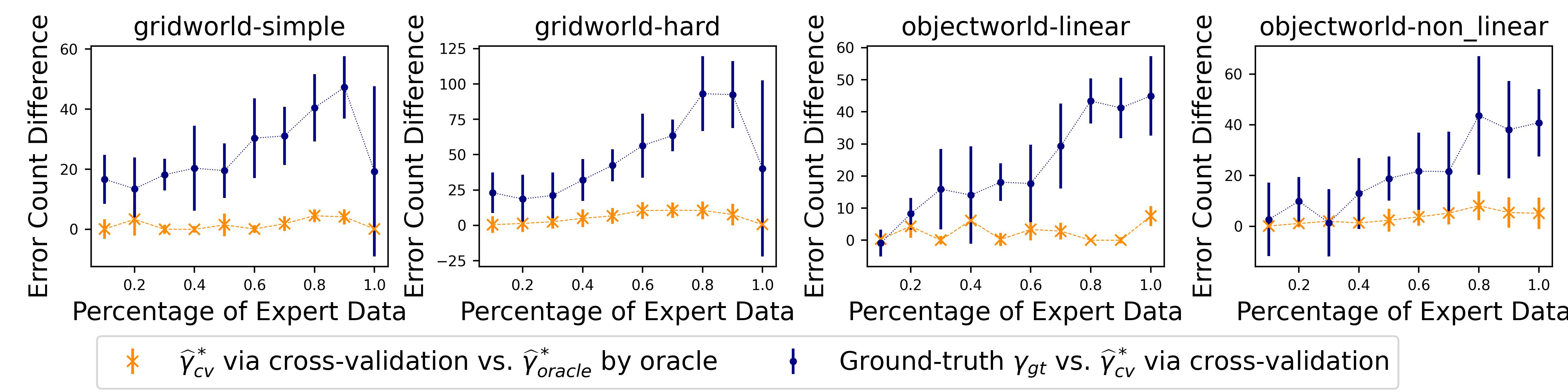}
\caption{The cross-validation results for LP-IRL on four tasks are shown. The \(x\)-axis represents the amount of expert data; the \( y \)-axis shows policy error count differences. We compare discount factors \(\optimalGammaCv\) (learned from cross-validation) and \(\optimalGammaCvAll\) (chosen by the oracle). Orange dots depict error differences between policies induced by \(\optimalGammaCv\) and \(\optimalGammaCvAll\); blue dots show differences between policies induced by \(\optimalGammaCv\) and the ground-truth \(\gtGamma\). The orange curves near zero indicate that cross-validation effectively selects \(\optimalGamma\), while the positive blue curves show that cross-validation consistently yields better policies than using \(\gtGamma\).}
\vspace{-25pt}
\label{fig:summary_cross_validation}
\Description{Cross-validation for LP-IRL}
\end{figure}

\textbf{Q.2 Optimal $\optimalGamma$s Vary with the Amount of Expert Data.}

Figures \ref{fig:gamma_vs_coverage} and \ref{fig:gamma_vs_coverage_maxent} show how the optimal discount factor \(\optimalGamma\) varies with increasing amounts of expert data for LP-IRL and MaxEnt-IRL, respectively. When data is scarce, \(\optimalGamma\) is high because the benefits of temporal reasoning outweigh overfitting concerns—the estimation error of the expert policy remains high regardless of overfitting. A large \(\optimalGamma\) enables extrapolation of actions to nearby unobserved states, reducing the first error term in Theorem \ref{theorem:final_bound} and improving overall error reduction.

With a moderate amount of data, \(\optimalGamma\) decreases: although expert policy estimation improves with more data, it is still limited. The overall error bound favors a smaller \(\optimalGamma\) to mitigate overfitting, reducing the second error term even at the cost of some temporal reasoning benefits. As data increases further, \(\optimalGamma\) rises again since overfitting becomes less significant, and larger values enhance temporal reasoning.

Overall, as expert data increases, error counts for \(\optimalGamma\) strictly decrease and remain below those for \(\gtGamma\), indicating that \(\optimalGamma\) enables IRL to learn more effectively from additional expert data.

\textbf{Q.3 Cross-Validation Effectively Selects Optimal $\optimalGamma$s.} 

Figures \ref{fig:summary_cross_validation} and \ref{fig:summary_cross_validation_maxent} summarize the cross-validation results for LP-IRL and MaxEnt-IRL, respectively. The performance discrepancy between policies induced by \(\optimalGammaCv\) and the oracle \(\optimalGammaCvAll\) (orange curves) is consistently near zero across all four tasks, indicating that cross-validation effectively selects \(\optimalGamma\) similar to the oracle. Moreover, the blue curves represent the error differences between policies induced by the ground-truth \(\gtGamma\) and \(\optimalGammaCv\), which are significantly higher than zero. This suggests that \(\optimalGammaCv\) yields better-performing policies than using the ground-truth \(\gtGamma\), confirming our theoretical findings in Theorem \ref{theorem:final_bound}.
\section{Conclusion}
In this paper, we present a theoretical analysis on IRL that unveils the potential of a reduced horizon in inducing a more expert-like policy, particularly in data-scarce situations. Our findings reveals an important insight on role of the horizon in IRL: it controls the complexity of the induced policy class, therefore reduces overfitting to the limited expert data. We, therefore, propose a more natural IRL function class that jointly learns reward-horizon pairs and empirically substantiate our analysis using a cross-validation extension for the existing IRL algorithms. As overfitting remains a challenge for IRL, especially with scarce expert data, we believe our findings offer valuable insights for the IRL community on better IRL formulations.


\begin{acks}
This research is supported in part by the National Research Foundation (NRF), Singapore and DSO National Laboratories under the AI Singapore Program (AISG Award No: AISG2-RP-2020-016), the National University of Singapore (AcRF Tier-1 grant A-8002616-00-00), the National Science Foundation (Grant No. IIS-2007076). Any opinions, findings and conclusions or recommendations expressed in this material are those of the author(s) and do not reflect the views of NRF Singapore or National Science Foundation.
\end{acks}


\balance
\bibliographystyle{ACM-Reference-Format} 
\bibliography{AAMAS_2025_sample}


\begin{thebibliography}{34}


\ifx \showCODEN    \undefined \def \showCODEN     #1{\unskip}     \fi
\ifx \showDOI      \undefined \def \showDOI       #1{#1}\fi
\ifx \showISBNx    \undefined \def \showISBNx     #1{\unskip}     \fi
\ifx \showISBNxiii \undefined \def \showISBNxiii  #1{\unskip}     \fi
\ifx \showISSN     \undefined \def \showISSN      #1{\unskip}     \fi
\ifx \showLCCN     \undefined \def \showLCCN      #1{\unskip}     \fi
\ifx \shownote     \undefined \def \shownote      #1{#1}          \fi
\ifx \showarticletitle \undefined \def \showarticletitle #1{#1}   \fi
\ifx \showURL      \undefined \def \showURL       {\relax}        \fi
\providecommand\bibfield[2]{#2}
\providecommand\bibinfo[2]{#2}
\providecommand\natexlab[1]{#1}
\providecommand\showeprint[2][]{arXiv:#2}

\bibitem[\protect\citeauthoryear{Abbeel and Ng}{Abbeel and Ng}{2004}]%
        {Apprenticeship}
\bibfield{author}{\bibinfo{person}{Pieter Abbeel} {and} \bibinfo{person}{Andrew Ng}.} \bibinfo{year}{2004}\natexlab{}.
\newblock \showarticletitle{{Apprenticeship Learning via Inverse Reinforcement Learning}}. In \bibinfo{booktitle}{\emph{Proc. Int. Conf. on Machine Learning}}.
\newblock


\bibitem[\protect\citeauthoryear{Amit, Meir, and Ciosek}{Amit et~al\mbox{.}}{2020}]%
        {amit2020discount}
\bibfield{author}{\bibinfo{person}{Ron Amit}, \bibinfo{person}{Ron Meir}, {and} \bibinfo{person}{Kamil Ciosek}.} \bibinfo{year}{2020}\natexlab{}.
\newblock \showarticletitle{Discount factor as a regularizer in reinforcement learning}. In \bibinfo{booktitle}{\emph{Proc. Int. Conf. on Machine Learning}}.
\newblock


\bibitem[\protect\citeauthoryear{Boularias, Kober, and Peters}{Boularias et~al\mbox{.}}{2011}]%
        {REIRL}
\bibfield{author}{\bibinfo{person}{Abdeslam Boularias}, \bibinfo{person}{Jens Kober}, {and} \bibinfo{person}{Jan Peters}.} \bibinfo{year}{2011}\natexlab{}.
\newblock \showarticletitle{{Relative Entropy Inverse Reinforcement Learning}}. In \bibinfo{booktitle}{\emph{Proc. Int. Conf. on Artificial Intelligence \& Statistics}}.
\newblock


\bibitem[\protect\citeauthoryear{Finn, Christiano, Abbeel, and Levine}{Finn et~al\mbox{.}}{2016a}]%
        {GAN-GCL}
\bibfield{author}{\bibinfo{person}{Chelsea Finn}, \bibinfo{person}{Paul Christiano}, \bibinfo{person}{Pieter Abbeel}, {and} \bibinfo{person}{Sergey Levine}.} \bibinfo{year}{2016}\natexlab{a}.
\newblock \showarticletitle{{A Connection Between Generative Adversarial Networks, Inverse Reinforcement Learning, and Energy-Based Models}}. In \bibinfo{booktitle}{\emph{Advances in Neural Information Processing Systems}}.
\newblock


\bibitem[\protect\citeauthoryear{Finn, Levine, and Abbeel}{Finn et~al\mbox{.}}{2016b}]%
        {GCL}
\bibfield{author}{\bibinfo{person}{Chelsea Finn}, \bibinfo{person}{Sergey Levine}, {and} \bibinfo{person}{Pieter Abbeel}.} \bibinfo{year}{2016}\natexlab{b}.
\newblock \showarticletitle{{Guided Cost Learning: Deep Inverse Optimal Control via Policy Optimization}}. In \bibinfo{booktitle}{\emph{Proc. Int. Conf. on Machine Learning}}.
\newblock


\bibitem[\protect\citeauthoryear{Fu, Luo, and Levine}{Fu et~al\mbox{.}}{2018}]%
        {AIRL}
\bibfield{author}{\bibinfo{person}{Justin Fu}, \bibinfo{person}{Katie Luo}, {and} \bibinfo{person}{Sergey Levine}.} \bibinfo{year}{2018}\natexlab{}.
\newblock \showarticletitle{{Learning Robust Rewards with Adversarial Inverse Reinforcement Learning}}. In \bibinfo{booktitle}{\emph{Proc. Int. Conf. on Learning Representations}}.
\newblock


\bibitem[\protect\citeauthoryear{Ghasemipour, Zemel, and Gu}{Ghasemipour et~al\mbox{.}}{2019}]%
        {FAIRL}
\bibfield{author}{\bibinfo{person}{Seyed Kamyar~Seyed Ghasemipour}, \bibinfo{person}{Richard Zemel}, {and} \bibinfo{person}{Shixiang Gu}.} \bibinfo{year}{2019}\natexlab{}.
\newblock \showarticletitle{{A Divergence Minimization Perspective on Imitation Learning Methods}}. In \bibinfo{booktitle}{\emph{Proc. Conference on Robot Learning}}.
\newblock


\bibitem[\protect\citeauthoryear{Ghosal, Zurek, Brown, and Dragan}{Ghosal et~al\mbox{.}}{2023}]%
        {ghosal2023effect}
\bibfield{author}{\bibinfo{person}{Gaurav~R Ghosal}, \bibinfo{person}{Matthew Zurek}, \bibinfo{person}{Daniel~S Brown}, {and} \bibinfo{person}{Anca~D Dragan}.} \bibinfo{year}{2023}\natexlab{}.
\newblock \showarticletitle{The effect of modeling human rationality level on learning rewards from multiple feedback types}. In \bibinfo{booktitle}{\emph{Proceedings of the AAAI Conference on Artificial Intelligence}}.
\newblock


\bibitem[\protect\citeauthoryear{Giwa and Lee}{Giwa and Lee}{2021}]%
        {giwa2021estimation}
\bibfield{author}{\bibinfo{person}{Babatunde~H Giwa} {and} \bibinfo{person}{Chi-Guhn Lee}.} \bibinfo{year}{2021}\natexlab{}.
\newblock \showarticletitle{Estimation of Discount Factor in a Model-Based Inverse Reinforcement Learning Framework}. In \bibinfo{booktitle}{\emph{Bridging the Gap Between AI Planning and Reinforcement Learning Workshop at ICAPS}}.
\newblock


\bibitem[\protect\citeauthoryear{Guo, Hu, and Zhang}{Guo et~al\mbox{.}}{2022}]%
        {guo2022theoretical}
\bibfield{author}{\bibinfo{person}{Xin Guo}, \bibinfo{person}{Anran Hu}, {and} \bibinfo{person}{Junzi Zhang}.} \bibinfo{year}{2022}\natexlab{}.
\newblock \showarticletitle{Theoretical guarantees of fictitious discount algorithms for episodic reinforcement learning and global convergence of policy gradient methods}. In \bibinfo{booktitle}{\emph{Proc. AAAI Conf. on Artificial Intelligence}}.
\newblock


\bibitem[\protect\citeauthoryear{Ho and Ermon}{Ho and Ermon}{2016}]%
        {GAIL}
\bibfield{author}{\bibinfo{person}{Jonathan Ho} {and} \bibinfo{person}{Stefano Ermon}.} \bibinfo{year}{2016}\natexlab{}.
\newblock \showarticletitle{{Generative Adversarial Imitation Learning}}. In \bibinfo{booktitle}{\emph{Advances in Neural Information Processing Systems}}.
\newblock


\bibitem[\protect\citeauthoryear{Hoshino, Ota, Kanezaki, and Yokota}{Hoshino et~al\mbox{.}}{2022}]%
        {opirl}
\bibfield{author}{\bibinfo{person}{Hana Hoshino}, \bibinfo{person}{Kei Ota}, \bibinfo{person}{Asako Kanezaki}, {and} \bibinfo{person}{Rio Yokota}.} \bibinfo{year}{2022}\natexlab{}.
\newblock \showarticletitle{{OPIRL: Sample Efficient Off-Policy Inverse Reinforcement Learning via Distribution Matching}}. In \bibinfo{booktitle}{\emph{Proc. IEEE Int. Conf. on Robotics \& Automation}}.
\newblock


\bibitem[\protect\citeauthoryear{Hu, Yang, Zhao, and Zhang}{Hu et~al\mbox{.}}{2022}]%
        {hu2022role}
\bibfield{author}{\bibinfo{person}{Hao Hu}, \bibinfo{person}{Yiqin Yang}, \bibinfo{person}{Qianchuan Zhao}, {and} \bibinfo{person}{Chongjie Zhang}.} \bibinfo{year}{2022}\natexlab{}.
\newblock \showarticletitle{On the role of discount factor in offline reinforcement learning}. In \bibinfo{booktitle}{\emph{Proc. Int. Conf. on Machine Learning}}.
\newblock


\bibitem[\protect\citeauthoryear{Jiang, Kulesza, Singh, and Lewis}{Jiang et~al\mbox{.}}{2015}]%
        {gamma}
\bibfield{author}{\bibinfo{person}{Nan Jiang}, \bibinfo{person}{A. Kulesza}, \bibinfo{person}{S. Singh}, {and} \bibinfo{person}{R. Lewis}.} \bibinfo{year}{2015}\natexlab{}.
\newblock \showarticletitle{{The Dependence of Effective Planning Horizon on Model Accuracy}}. In \bibinfo{booktitle}{\emph{Proc. Int. Conf. on Autonomous Agents \& Multiagent Systems}}.
\newblock


\bibitem[\protect\citeauthoryear{Ke, Choudhury, Barnes, Sun, Lee, and Srinivasa}{Ke et~al\mbox{.}}{2020}]%
        {f_D}
\bibfield{author}{\bibinfo{person}{Liyiming Ke}, \bibinfo{person}{Sanjiban Choudhury}, \bibinfo{person}{Matt Barnes}, \bibinfo{person}{Wen Sun}, \bibinfo{person}{Gilwoo Lee}, {and} \bibinfo{person}{Siddhartha Srinivasa}.} \bibinfo{year}{2020}\natexlab{}.
\newblock \showarticletitle{{Imitation Learning as f-Divergence Minimization}}. In \bibinfo{booktitle}{\emph{Algorithmic Foundations of Robotics XIV---Proc. Int. Workshop on the Algorithmic Foundations of Robotics (WAFR)}}.
\newblock


\bibitem[\protect\citeauthoryear{Laidlaw, Russell, and Dragan}{Laidlaw et~al\mbox{.}}{2023}]%
        {laidlaw2023bridging}
\bibfield{author}{\bibinfo{person}{Cassidy Laidlaw}, \bibinfo{person}{Stuart~J Russell}, {and} \bibinfo{person}{Anca Dragan}.} \bibinfo{year}{2023}\natexlab{}.
\newblock \showarticletitle{Bridging rl theory and practice with the effective horizon}.
\newblock \bibinfo{journal}{\emph{Advances in Neural Information Processing Systems}}  \bibinfo{volume}{36} (\bibinfo{year}{2023}), \bibinfo{pages}{58953--59007}.
\newblock


\bibitem[\protect\citeauthoryear{Lee, Isele, Theodorou, and Bae}{Lee et~al\mbox{.}}{2022}]%
        {CostMPC}
\bibfield{author}{\bibinfo{person}{Keuntaek Lee}, \bibinfo{person}{David Isele}, \bibinfo{person}{Evangelos~A. Theodorou}, {and} \bibinfo{person}{Sangjae Bae}.} \bibinfo{year}{2022}\natexlab{}.
\newblock \showarticletitle{{Spatiotemporal Costmap Inference for MPC via Deep Inverse Reinforcement Learning}}. In \bibinfo{booktitle}{\emph{IEEE Robotics \& Automation Letters}}.
\newblock


\bibitem[\protect\citeauthoryear{Levine, Popovic, and Koltun}{Levine et~al\mbox{.}}{2011}]%
        {GPIRL}
\bibfield{author}{\bibinfo{person}{Sergey Levine}, \bibinfo{person}{Zoran Popovic}, {and} \bibinfo{person}{Vladlen Koltun}.} \bibinfo{year}{2011}\natexlab{}.
\newblock \showarticletitle{{Nonlinear Inverse Reinforcement Learning with Gaussian Processes}}. In \bibinfo{booktitle}{\emph{Advances in Neural Information Processing Systems}}.
\newblock


\bibitem[\protect\citeauthoryear{MacGlashan and Littman}{MacGlashan and Littman}{2015}]%
        {BetweenImitationAndIntention}
\bibfield{author}{\bibinfo{person}{James MacGlashan} {and} \bibinfo{person}{Michael~L. Littman}.} \bibinfo{year}{2015}\natexlab{}.
\newblock \showarticletitle{{Between Imitation and Intention Learning}}. In \bibinfo{booktitle}{\emph{Proc. Int. Jnt. Conf. on Artificial Intelligence}}.
\newblock


\bibitem[\protect\citeauthoryear{Mattingley, Wang, and Boyd}{Mattingley et~al\mbox{.}}{2011}]%
        {mattingley2011receding}
\bibfield{author}{\bibinfo{person}{Jacob Mattingley}, \bibinfo{person}{Yang Wang}, {and} \bibinfo{person}{Stephen Boyd}.} \bibinfo{year}{2011}\natexlab{}.
\newblock \showarticletitle{Receding horizon control}.
\newblock \bibinfo{journal}{\emph{IEEE Control Systems Magazine}} \bibinfo{volume}{31}, \bibinfo{number}{3} (\bibinfo{year}{2011}), \bibinfo{pages}{52--65}.
\newblock


\bibitem[\protect\citeauthoryear{Metelli, Ramponi, Concetti, and Restelli}{Metelli et~al\mbox{.}}{2021}]%
        {pmlr-v139-metelli21a}
\bibfield{author}{\bibinfo{person}{Alberto~Maria Metelli}, \bibinfo{person}{Giorgia Ramponi}, \bibinfo{person}{Alessandro Concetti}, {and} \bibinfo{person}{Marcello Restelli}.} \bibinfo{year}{2021}\natexlab{}.
\newblock \showarticletitle{{Provably Efficient Learning of Transferable Rewards}}. In \bibinfo{booktitle}{\emph{Proc. Int. Conf. on Machine Learning}}.
\newblock


\bibitem[\protect\citeauthoryear{Ng, Harada, and Russell}{Ng et~al\mbox{.}}{1999}]%
        {ng1999policy}
\bibfield{author}{\bibinfo{person}{Andrew~Y Ng}, \bibinfo{person}{Daishi Harada}, {and} \bibinfo{person}{Stuart Russell}.} \bibinfo{year}{1999}\natexlab{}.
\newblock \showarticletitle{{Policy Invariance Under Reward Transformations: Theory and Application to Reward Shaping}}. In \bibinfo{booktitle}{\emph{Proc. Int. Conf. on Machine Learning}}.
\newblock


\bibitem[\protect\citeauthoryear{Ng and Russell}{Ng and Russell}{2000}]%
        {AndrewNg}
\bibfield{author}{\bibinfo{person}{Andrew~Y. Ng} {and} \bibinfo{person}{Stuart~J. Russell}.} \bibinfo{year}{2000}\natexlab{}.
\newblock \showarticletitle{{Algorithms for Inverse Reinforcement Learning}}. In \bibinfo{booktitle}{\emph{Proc. Int. Conf. on Machine Learning}}.
\newblock


\bibitem[\protect\citeauthoryear{Ni, Sikchi, Wang, Gupta, Lee, and Eysenbach}{Ni et~al\mbox{.}}{2020}]%
        {fIRL}
\bibfield{author}{\bibinfo{person}{Tianwei Ni}, \bibinfo{person}{Harshit Sikchi}, \bibinfo{person}{Yufei Wang}, \bibinfo{person}{Tejus Gupta}, \bibinfo{person}{Lisa Lee}, {and} \bibinfo{person}{Benjamin Eysenbach}.} \bibinfo{year}{2020}\natexlab{}.
\newblock \showarticletitle{{f-IRL: Inverse Reinforcement Learning via State Marginal Matching}}. In \bibinfo{booktitle}{\emph{Proc. Conference on Robot Learning}}.
\newblock


\bibitem[\protect\citeauthoryear{Osa, Pajarinen, Neumann, Bagnell, Abbeel, Peters, et~al\mbox{.}}{Osa et~al\mbox{.}}{2018}]%
        {IM_book}
\bibfield{author}{\bibinfo{person}{Takayuki Osa}, \bibinfo{person}{Joni Pajarinen}, \bibinfo{person}{Gerhard Neumann}, \bibinfo{person}{J~Andrew Bagnell}, \bibinfo{person}{Pieter Abbeel}, \bibinfo{person}{Jan Peters}, {et~al\mbox{.}}} \bibinfo{year}{2018}\natexlab{}.
\newblock \showarticletitle{{An Algorithmic Perspective on Imitation Learning}}.
\newblock \bibinfo{journal}{\emph{Foundations and Trends{\textregistered} in Robotics}} \bibinfo{volume}{7}, \bibinfo{number}{1-2} (\bibinfo{year}{2018}), \bibinfo{pages}{1--179}.
\newblock


\bibitem[\protect\citeauthoryear{Petrik and Scherrer}{Petrik and Scherrer}{2008}]%
        {petrik2008biasing}
\bibfield{author}{\bibinfo{person}{Marek Petrik} {and} \bibinfo{person}{Bruno Scherrer}.} \bibinfo{year}{2008}\natexlab{}.
\newblock \showarticletitle{Biasing approximate dynamic programming with a lower discount factor}. In \bibinfo{booktitle}{\emph{Advances in neural information processing systems}}.
\newblock


\bibitem[\protect\citeauthoryear{Pirotta and Restelli}{Pirotta and Restelli}{2016}]%
        {IRL_PG}
\bibfield{author}{\bibinfo{person}{Matteo Pirotta} {and} \bibinfo{person}{Marcello Restelli}.} \bibinfo{year}{2016}\natexlab{}.
\newblock \showarticletitle{{Inverse Reinforcement Learning through Policy Gradient Minimization}}. In \bibinfo{booktitle}{\emph{Proc. AAAI Conf. on Artificial Intelligence}}.
\newblock


\bibitem[\protect\citeauthoryear{Ramachandran and Amir}{Ramachandran and Amir}{2007}]%
        {BIRL}
\bibfield{author}{\bibinfo{person}{Deepak Ramachandran} {and} \bibinfo{person}{Eyal Amir}.} \bibinfo{year}{2007}\natexlab{}.
\newblock \showarticletitle{{Bayesian Inverse Reinforcement Learning}}. In \bibinfo{booktitle}{\emph{Proc. Int. Jnt. Conf. on Artificial Intelligence}}.
\newblock


\bibitem[\protect\citeauthoryear{Ramponi, Drappo, and Restelli}{Ramponi et~al\mbox{.}}{2020}]%
        {IRL_policy_based_learner}
\bibfield{author}{\bibinfo{person}{Giorgia Ramponi}, \bibinfo{person}{Gianluca Drappo}, {and} \bibinfo{person}{Marcello Restelli}.} \bibinfo{year}{2020}\natexlab{}.
\newblock \showarticletitle{{Inverse Reinforcement Learning from a Gradient-Based Learner}}.
\newblock \bibinfo{journal}{\emph{Advances in Neural Information Processing Systems}}  \bibinfo{volume}{33} (\bibinfo{year}{2020}), \bibinfo{pages}{2458--2468}.
\newblock


\bibitem[\protect\citeauthoryear{Ratliff, Bagnell, and Zinkevich}{Ratliff et~al\mbox{.}}{2006}]%
        {maxmargin_IRL}
\bibfield{author}{\bibinfo{person}{Nathan~D. Ratliff}, \bibinfo{person}{J.~Andrew Bagnell}, {and} \bibinfo{person}{Martin~A. Zinkevich}.} \bibinfo{year}{2006}\natexlab{}.
\newblock \showarticletitle{{Maximum Margin Planning}}. In \bibinfo{booktitle}{\emph{Proc. Int. Conf. on Machine Learning}}.
\newblock


\bibitem[\protect\citeauthoryear{Spencer, Choudhury, Venkatraman, Ziebart, and Bagnell}{Spencer et~al\mbox{.}}{2021}]%
        {FeedbackII}
\bibfield{author}{\bibinfo{person}{Jonathan Spencer}, \bibinfo{person}{Sanjiban Choudhury}, \bibinfo{person}{Arun Venkatraman}, \bibinfo{person}{Brian~D. Ziebart}, {and} \bibinfo{person}{J.~Andrew Bagnell}.} \bibinfo{year}{2021}\natexlab{}.
\newblock \showarticletitle{{Feedback in Imitation Learning: The Three Regimes of Covariate Shift}}.
\newblock \bibinfo{journal}{\emph{ArXiv}}  \bibinfo{volume}{abs/2102.02872} (\bibinfo{year}{2021}).
\newblock


\bibitem[\protect\citeauthoryear{Wulfmeier, Ondruska, and Posner}{Wulfmeier et~al\mbox{.}}{2016}]%
        {DeepMaxEnt}
\bibfield{author}{\bibinfo{person}{Markus Wulfmeier}, \bibinfo{person}{Peter Ondruska}, {and} \bibinfo{person}{Ingmar Posner}.} \bibinfo{year}{2016}\natexlab{}.
\newblock \showarticletitle{{Maximum Entropy Deep Inverse Reinforcement Learning}}. In \bibinfo{booktitle}{\emph{Advances in Neural Information Processing Systems}}.
\newblock


\bibitem[\protect\citeauthoryear{Xu, Gao, and Hsu}{Xu et~al\mbox{.}}{2022}]%
        {RHIRL}
\bibfield{author}{\bibinfo{person}{Yiqing Xu}, \bibinfo{person}{Wei Gao}, {and} \bibinfo{person}{David Hsu}.} \bibinfo{year}{2022}\natexlab{}.
\newblock \showarticletitle{{Receding Horizon Inverse Reinforcement Learning}}. In \bibinfo{booktitle}{\emph{Advances in Neural Information Processing Systems}}.
\newblock


\bibitem[\protect\citeauthoryear{Ziebart, Maas, Bagnell, and Dey}{Ziebart et~al\mbox{.}}{2008}]%
        {maxentirl}
\bibfield{author}{\bibinfo{person}{Brian~D. Ziebart}, \bibinfo{person}{Andrew Maas}, \bibinfo{person}{J.~Andrew Bagnell}, {and} \bibinfo{person}{Anind~K. Dey}.} \bibinfo{year}{2008}\natexlab{}.
\newblock \showarticletitle{{Maximum Entropy Inverse Reinforcement Learning}}. In \bibinfo{booktitle}{\emph{Proc. AAAI Conf. on Artificial Intelligence}}.
\newblock


\end{thebibliography}

\newpage

\appendix
\section{Complexity Measure of Policy} \label{appendix:complexity_measure}

\begin{definition}[Reward and Policy Equivalence]
For an MDP $\MDP = ( \stateSpace, \actionSpace, \transition, ., \cGamma )$, we define two bounded reward functions $\cR$ and $\cR'$ to be equivalent, i.e. $\cR \equiv \cR'$, if and only if they induce the same set of optimal policies.
\end{definition}

\begin{lemma}[Potential-based Reward Shaping \cite{ng1999policy}]\label{reward_shaping}
For an MDP $\MDP = ( \stateSpace, \actionSpace, \transition, ., \cGamma)$, two bounded reward functions $\cR$ and $\cR'$ are equivalent, i.e. $\cR \equiv \cR'$, if and only if there exists a bounded potential function $\potential: \stateSpace \to \mathbb{R}$ such that for all $s, s' \in \stateSpace$ and $a \in \actionSpace$:
\begin{equation}
    \cR'(s, a, s') = \cR(s, a, s') + \cGamma \potential(s') - \potential(s)
\end{equation}
\end{lemma}

\begin{remark}
We extend the potential-based reward shaping presented in lemma \ref{reward_shaping}, which utilizes a reward function parameterized by $\cR(s, a, s')$, to one that does not rely on the subsequent state, namely, $\cR(s, a)$. For an MDP $\MDP = ( \stateSpace, \actionSpace, \transition, ., \cGamma )$, two bounded reward functions $\cR(s, a)$ and $\cR'(s, a)$ are equivalent, if and only if there exists a bounded potential function $\potential: \stateSpace \to \mathbb{R}$ such that for all $s, s' \in \stateSpace$ and $a \in \actionSpace$:
\begin{equation}
    \cR'(s, a)= \cR(s, a) +  \cGamma  \sum_{s'} \transition(s'| a, s)\potential(s') -\potential(s)
\end{equation}
\end{remark}
\begin{proof}
Consider the reward function $\cR(s, a, s')$, its $\cR(s, a)$ counterpart is defined as follows:
\begin{equation}
    \cR(s, a) = \sum_{s'} \transition(s'| a, s) \cR(s, a, s')
\end{equation}
Now, we consider $\cR'(s, a, s') \equiv \cR(s, a, s')$ by potential-based shaping with potential function $\potential$, we have the following reward-equivalent shaping for $\cR'(s, a)$:
\begin{align}
    \cR'(s, a) & = \sum_{s'} \transition(s'|a, s) \cR'(s, a, s') \nonumber\\
    = & \sum_{s'} \transition(s'| a, s)(\cR(s, a, s') + \cGamma \potential(s') - \potential(s)) \nonumber\\
    = & \sum_{s'} \transition(s'| a, s')\cR(s, a, s') + \cGamma \sum_{s'} \transition(s'| a, s)\potential(s') - \potential(s) \sum_{s'} \transition(s'| a, s)\nonumber\\
    = & \cR(s, a) +  \cGamma  \sum_{s'} \transition(s'| a, s)\potential(s') -\potential(s)
\end{align}
\end{proof}

\begin{proof}
\textit{Proof of theorem 1, claim 1}. Given $\cGamma, \cGamma' \in [0, 1), \cGamma < \cGamma'$, we prove if $\pi \in \Pi_{\cGamma}$, then $\pi \in \Pi_{\cGamma'}$ as well. Formally, let $\pi$ be the optimal policy in MDP $\MDP = ( \stateSpace, \actionSpace, \transition, \cR, \cGamma )$, we construct a reward function $\cR' \in F_{\cR}$ such that $\pi$ is still optimal in MDP $( \stateSpace, \actionSpace, \transition, \cR', \cGamma' )$.

Given policy $\pi$, let $\piTransition{\pi}$ be the transition function matrix of size $|\stateSpace| \times |\stateSpace|$ such that $[\piTransition{\pi}](s, s') = \transition(s'| \pi(s), s)$, and $\piR{\pi}$ be the reward vector of size $|\stateSpace|$ such that $[\piR{\pi}](s) = \cR(s, \pi(s))$. We can write the value function $\VFunction{\pi}{\cR}{ \cGamma}$ of policy $\pi$ evaluated under the reward function $R$ and discount factor $\cGamma$ as follows:
\begin{equation}
    \VFunction{\pi}{\cR}{\cGamma} = \piR{\pi} + \cGamma \piTransition{\pi}\VFunction{\pi}{\cR}{\cGamma}
\end{equation}

Next, we apply a potential-based shaping to the original reward function $\piR{\pi}$. Let $\piR{'\pi}$ be the shaped reward vector of size $|\stateSpace|$ such that $[\piR{'\pi}](s) = \cR'(s, \pi(s))$, we have:
\begin{equation}
    \piR{'\pi} = \piR{\pi} + \cGamma \piTransition{\pi}\potential - \potential
\end{equation}
where $\potential$ is the potential vector defined as follows:
\begin{equation}
    \potential = \frac{\cGamma' - \cGamma}{\cGamma}\VFunction{\pi}{\cR'}{\cGamma}
\end{equation} 

With this shaped reward function $\cR'$ and $\cGamma$, we have the new value function:
\begin{align}
    \VFunction{\pi}{\cR'}{\cGamma} = & \piR{'\pi} + \cGamma \piTransition{\pi}\VFunction{\pi}{\cR'}{\cGamma} \nonumber\\
     = & \piR{\pi} + \cGamma \piTransition{\pi}\potential - \potential + \cGamma \piTransition{\pi}\VFunction{\pi}{\cR'}{\cGamma} \nonumber\\
     = & \piR{\pi} + \cGamma \frac{\cGamma' - \cGamma}{\cGamma}\piTransition{\pi}\VFunction{\pi}{\cR'}{\cGamma} - \frac{\cGamma' - \cGamma}{\cGamma}\VFunction{\pi}{\cR'}{\cGamma} + \cGamma \piTransition{\pi}\VFunction{\pi}{\cR'}{\cGamma} \nonumber\\
     = & \piR{\pi} + \cGamma' \piTransition{\pi}\VFunction{\pi}{\cR'}{\cGamma} - \frac{\cGamma' - \cGamma}{\cGamma}\VFunction{\pi}{\cR'}{\cGamma} \nonumber 
\end{align}
Rearranging, we have:
\begin{equation}
     \VFunction{\pi}{\cR'}{\cGamma} = (I - \cGamma' \piTransition{\pi} + \frac{\cGamma' - \cGamma}{\cGamma})^{-1}\piR{\pi} 
\end{equation}
Since $\cR \equiv \cR'$, their respective value function $\VFunction{\pi}{\cR}{\cGamma}$ and $\VFunction{\pi}{\cR'}{\cGamma}$ induce the same set of optimal policies. 
We emphasize that $\cR'$ is not necessary from $F_{\cR}$. 
To prove that any optimal policy $\pi \in \Pi_{\cGamma}$ is still optimal in $\Pi_{\cGamma'}$, we construct a $\hR \in F_{\cR}$ which can make its value function $\VFunction{\pi}{\hR}{\cGamma'}$ evaluated with larger $\cGamma'$ the same as $\VFunction{\pi}{\cR'}{\cGamma} $. 

We construct $\hR$ as follows:
\begin{equation} \label{equ:construct_reward}
    \piRh = \piR{\pi} - \frac{\cGamma'-\cGamma}{\cGamma}\VFunction{\pi}{\hR}{\cGamma'}
\end{equation}
The value function $ \VFunction{\pi}{\hR}{\cGamma'}$ for $\hR$ and $\cGamma'$ is:
\begin{align}
    \VFunction{\pi}{\hR}{\cGamma'} & = \piRh + \cGamma' \piTransition{\pi}\VFunction{\pi}{\hR}{\cGamma'} \\
    & =  R^{\pi} - \frac{\cGamma'-\cGamma}{\cGamma}\VFunction{\pi}{\hR}{\cGamma'} +\cGamma' \piTransition{\pi}\VFunction{\pi}{\hR}{\cGamma'}
\end{align}

Rearrange, we have:
\begin{align}
     \VFunction{\pi}{\hR}{\cGamma'} & = (I - \cGamma' \piTransition{\pi} + \frac{\cGamma' - \cGamma}{\cGamma})^{-1}\piR{\pi} \\
     & = \VFunction{\pi}{\cR'}{\cGamma}
\end{align}

It suffices to show that the construction for $\hR$ in equation \eqref{equ:construct_reward} satisfies $\hR \in F_{\cR}$. We now write the construction for $\hR(s, a)$ for every $(s, a)$ pair:
\begin{equation}
    \hR(s, a) = \cR(s, a) - \frac{\cGamma'-\cGamma}{\cGamma}\VFunction{\pi}{\hR}{\cGamma'}(s)
\end{equation}
We notice that the second term is a factor of the value function, which only depends on the current state $s$. Therefore, $\cR(s, a^*) > \cR(s, a)$ iff $\hR(s, a^*) > \hR(s, a)$ for all $s \in S$. That is, $\hR \in F_{\cR}$.
\end{proof}

\begin{proof}
Proof of theorem 1, claim 2. $\cGamma=0$ is the special case where the planning only performs one-step look ahead and optimize the immediate reward greedily. 
When $\cGamma=0$, the planning objective reduces to
\begin{equation}
    \optimalPolicy{\cR}{\cGamma=0}  = \underset{\pi \in \Pi}{\arg\max}\mathop{\mathbb{E}}_{a_t \sim \pi(s_t)} [\cR(s_t, a_t)]
\end{equation}
Given the assumption that $\forall s \in \stateSpace$, $\underset{a \in \actionSpace}{\arg \max}\; \cR(s, a)$ is unique, then $\pi^*$ is also unique and $|\policyComplexity{0}|$ = 1.
\end{proof}

\begin{proof}
Proof of theorem 1, claim 3. This proof is by construction. We consider a fully connected state space with the transition function $P(s'| s, a)$ defined below. Recall that for each $s \in \stateSpace$, there exists a unique $\optimalA$ that maximizes $\cR(s, a)$, we first define the transition for each state for taking this $\optimalA$:
\begin{equation}\label{transition1}
 \forall s \in \stateSpace, \;\; P(s'| s, \optimalA) =
    \begin{cases}
       1 & \text{if $s' = s$}\\
    0 & \text{otherwise}
    \end{cases}       
\end{equation}
For all other actions $a \in \actionSpace/\{\optimalA\}$ in state $s \in \stateSpace$, the transition function is defined as follows:
\begin{equation}\label{transition2}
  \;\; P(s'|s, a) = \frac{1}{|\stateSpace|-1}, \forall s' \in \stateSpace/ \{s\}
\end{equation}
The given transition model corresponds to a fully connected state space where $(s, \optimalA)$ creates a self-loop, and any other action has an equal probability of transitioning to different states. We construct $\cR_{\optimalS} \in F_{\cR}$ in this manner: for a state $\optimalS \in \stateSpace$, let $\cR_{\optimalS}(\optimalS, \optimalAS) > 2|\stateSpace|\cR_{\optimalS}(s, \optimalA)$ for any other state $s \in \stateSpace / \{\optimalS\}$. Consider an arbitrary policy $\pi$ with the constraints $\pi(\optimalS) = \optimalAS$ and $\pi(s) \neq \optimalA$ for other states. We demonstrate that this policy $\pi$ is optimal in $\cR_{\optimalS}$ and $\transition$. The optimality of $\pi$ at $\optimalS$ is apparent since this state is absorbing and $\pi$ selects the action maximizing immediate reward. For any other state $s$, we show that $\pi$ is optimal by calculating the optimal Q-value of $(s, \pi(s))$ compared to any other action $a$. Remember that for $s \neq s^*$, we constrain $\pi$ not to choose $\optimalA$, so the alternative choice of $a$ is to precisely select $\optimalA$. Therefore, we have:

 \begin{align}
     \optimalQ(s, \pi(s)) = & \cR(s, \pi(s)) + \cGamma (\frac{1}{|\stateSpace|-1}\frac{1}{1-\cGamma}\cR(\optimalS,\optimalAS) \\
     & + \frac{1}{|\stateSpace|-1}\sum_{s' \in \stateSpace/\{\optimalS\}}Q(s', \pi(s')) \\
      Q^*(s, \optimalA) = & \cR(s, \optimalA) + \frac{\cGamma}{1-\cGamma}\cR(s, \optimalA) 
 \end{align}
 
 We have:
\begin{align}
     \optimalQ(s, \pi(s)) & > \cR(s, \pi(s)) + \frac{\cGamma}{1-\cGamma}\frac{1}{|\stateSpace|-1}\cR(\optimalS, \optimalAS) \nonumber \\
     & > \cR(s, \pi(s)) + \frac{\cGamma}{1-\cGamma}\frac{2|\stateSpace|}{|\stateSpace|-1}\cR(s, \optimalA) \nonumber\\
     & > \cR(s, \pi(s)) + \frac{\cGamma}{1-\cGamma}2\cR(s, \optimalA)
\end{align}
Since $2\cR(s, \optimalA) - \cR(s, \optimalA) =  \cR(s, \optimalA) > 0$, and as $\cGamma$ approaches one, $\frac{\cGamma}{1-\cGamma}$ tends to infinity, so for sufficiently large $\cGamma$ we can guarantee that $\optimalQ(s, \pi(s)) \geq \optimalQ(s, a)$. Given each $\optimalS$ and its corresponding $\cR_{\optimalS}$, under our constraints for $\pi$, there are $(|\actionSpace|-1)^{|\stateSpace|-1}$ such policies. In addition, since the choice of $\optimalS$ is arbitrary, we can form $|\stateSpace|$ of such $\cR_{\optimalS}$, therefore, the total number of such policy is $(|
\actionSpace|-1)^{|\stateSpace|-1}|\stateSpace|$. 
\end{proof}
\section{Proofs for Section \ref{sec:main_result}}
\subsection{Proof for Lemma \ref{lemma:feasible_reward}} \label{appendix:proof_feasible_reward}
\begin{proof}
This proof is adapted from the proof of Lemma B.1 and Lemma 3.2 in \citet{pmlr-v139-metelli21a}. Recall that we use the advantage function to derive two conditions such that the expert policy $\expertPolicy$ is optimal under the reward function $\learnedR$ and $\learnedGamma$. Specifically,
\begin{align}
    & \QFunction{\expertPolicy}{\learnedR}{\learnedGamma}(s, a) - \VFunction{\expertPolicy}{\learnedR}{\learnedGamma}(s)  = 0 \;\;\;\;\;\;\;\text{if}\; \expertPolicy(a|s) > 0, \label{equ:cond_1}\\
     &\QFunction{\expertPolicy}{\learnedR}{\learnedGamma}(s, a) - \VFunction{\expertPolicy}{\learnedR}{\learnedGamma}(s) \leq 0 \;\;\;\;\;\;\;\text{if}\; \expertPolicy(a|s) = 0. \label{equ:cond_2}
\end{align}

Consider an IRL problem $\IRL = (\partialMDP, \expertPolicy)$. A Q-function satisfies the specified conditions if and only if there exist $\zeta \in \mathbb{R}^{\stateSpace \times \actionSpace}_{\geq 0}$ and $V \in \mathbb{R}^{|\stateSpace|}$ such that:

\begin{equation}\label{equ:Q-function}
    Q_{\learnedR, \learnedGamma} = - \expertFilterC{\expertPolicy}\zeta +E V.
\end{equation}

Given that $\expertPolicy \expertFilterC{\expertPolicy} = \mathbf{0}_{\stateSpace}$ and $\expertPolicy E = I_S$, the corresponding value function is $V_{\learnedR, \learnedGamma} = \expertPolicy Q_{\learnedR, \learnedGamma} =V$. For any $s\in \stateSpace$ and $a \in \actionSpace$ with $\expertPolicy(a|s) > 0$, we obtain $Q_{\learnedR, \learnedGamma}(s, a) = V(s) = V_{\learnedR, \learnedGamma}(s)$. This establishes the first condition in equation \ref{equ:cond_1}. If $a \in \actionSpace$ has $\expertPolicy(a|s) = 0$, then $Q_{\learnedR, \learnedGamma}(s, a) = - \zeta(s, a) + V(s) = - \zeta(s, a) + V_{\learnedR, \learnedGamma}(s) \leq V_{\learnedR, \learnedGamma}(s)$. This verifies the second condition in equation \ref{equ:cond_2}. Conversely, if $Q_{\learnedR, \learnedGamma}$ fulfills the two conditions, we set $V = V_{\learnedR, \learnedGamma}$ and $\zeta = EV_{\learnedR, \learnedGamma}-Q_{\learnedR, \learnedGamma} \leq 0$.

Next, recall that $Q_{\learnedR, \learnedGamma} = \learnedR + \learnedGamma \transition \expertPolicy Q_{\learnedR, \learnedGamma}$. The Q-function can be written as the fixed point of the above Bellman equation: $Q_{\learnedR, \learnedGamma} = (I_{\stateSpace\times \actionSpace} - \learnedGamma \transition \expertPolicy )^{-1} \learnedR$, and for $\learnedGamma<1$, the matrix is invertible. In other words, with fixed $\expertPolicy$, $\transition$, and $\learnedGamma<1$, there is a one-to-one correspondence between Q-functions and rewards. From equation \ref{equ:Q-function}, we obtain:

\begin{align}
    \learnedR & = (I_{\stateSpace\times \actionSpace} - \learnedGamma \transition \expertPolicy ) ( - \expertFilterC{\expertPolicy}\zeta +EV) \nonumber\\
   & =  - \expertFilterC{\expertPolicy}\zeta + \learnedGamma \transition \expertPolicy \expertFilterC{\expertPolicy}\zeta + (E - \learnedGamma \transition \expertPolicy E)V\nonumber\\
   & =  - \expertFilterC{\expertPolicy}\zeta + (E - \learnedGamma \transition)V,
\end{align}

since $\expertPolicy \expertFilterC{\expertPolicy} = \mathbf{0}_S$ and $\expertPolicy E = I_S$.
\end{proof}
\subsection{Proof for Theorem \ref{theorem:reward_expert}}
\label{appendix:proof_reward_error}
\begin{proof} 

This proof is adapted from Theorem 3.1 in \citet{pmlr-v139-metelli21a}. Note that $\gtR$ in Theorem \ref{theorem:reward_expert} is the ground-truth reward function and has the corresponding ground-truth discount factor $\gtGamma$. Using Lemma \ref{lemma:feasible_reward}, we express reward functions $\gtR \in \feasibleSet{\IRL}$ and $\learnedR \in \feasibleSet{\approxIRL}$ as:

\begin{align}
    \gtR & = -\expertFilterC{\expertPolicy}\zeta + (E - \gtGamma \transition)V, \\
    \learnedR & = -\expertFilterC{\approxExpertPolicy}\hat{\zeta} + (E - \learnedGamma \transition)\hat{V},
\end{align}

where $V, \hat{V} \in \mathbb{R}^{\stateSpace}$ and $\zeta, \hat{\zeta} \in \mathbb{R}^{\stateSpace \times \actionSpace}_{\geq 0}$. To find the existence of $\learnedR \in \feasibleSet{\approxIRL}$, we choose $\hat{V} =(E - \learnedGamma \transition)^{-1} (E - \gtGamma \transition)V$ and $\hat{\zeta} = \expertFilterC{\expertPolicy}\zeta$. Then:

\begin{align}
    \gtR - \learnedR  = & - (\expertFilterC{\expertPolicy}\zeta - \expertFilterC{\approxExpertPolicy}\expertFilterC{\expertPolicy}\zeta) + (E - \gtGamma \transition)V - (E - \learnedGamma \transition)(E - \learnedGamma \transition)^{-1} (E - \gtGamma \transition)V \nonumber \\
    = & - (I_{\stateSpace \times \actionSpace} - \expertFilterC{\approxExpertPolicy})\expertFilterC{\expertPolicy}\zeta \nonumber \\
    = & - \expertFilter{\approxExpertPolicy}\expertFilterC{\expertPolicy}\zeta
\end{align}

As the expert-filter-complement $\expertFilterC{\approxExpertPolicy}$ is linear and sums with the expert-filter to unity, i.e., $\expertFilter{\approxExpertPolicy} + \expertFilterC{\approxExpertPolicy} = I_{\stateSpace \times \actionSpace}$, we get:

\begin{equation}
    \lvert \gtR - \learnedR \rvert \leq\expertFilter{\approxExpertPolicy}\expertFilterC{\expertPolicy}\zeta.
\end{equation}

Finally, we obtain $\left\lVert\zeta\right\rVert_{\infty} \leq \frac{R_{\max}}{1-\gtGamma}$ by using the condition $\left\lVert\gtR\right\rVert \leq R_{\max}$.
\end{proof}

\subsection{Proof of Theorem \ref{theorem:value_bound}} \label{appendix:proof_value_error}
To prove Theorem \ref{theorem:value_bound}, we need three lemmas.

\begin{lemma}[lemma 1 from \cite{gamma}] \label{lemma1}
For any MDP $\MDP$ with rewards in $[0, R_{max}]$, $\forall \pi : \stateSpace \to \actionSpace$ and $\learnedGamma \leq \gtGamma$, 
\begin{equation}
    \VFunction{\pi}{\gtR}{\learnedGamma} \leq \VFunction{\pi}{\gtR}{\gtGamma} \leq   \VFunction{\pi}{\gtR}{\learnedGamma} + \frac{\gtGamma - \learnedGamma}{(1-\gtGamma)(1-\learnedGamma)}R_{max}
\end{equation}
Hence, we have the following upper bound:
\begin{equation}
    \left\lVert\VFunction{\optimalPolicy{\gtR}{\gtGamma}}{\gtR}{\gtGamma}
    - \VFunction{\optimalPolicy{\gtR}{\gtGamma}}{\gtR}{\learnedGamma}\right\rVert_{\infty} \leq \frac{\gtGamma - \learnedGamma}{(1-\gtGamma)(1-\learnedGamma)}R_{max}
\end{equation}
\end{lemma}

Intuitively, Lemma \ref{lemma1} measures the performance discrepancy of a policy $\pi$ when evaluated under two different discount factors: $\learnedGamma$ and $\gtGamma$.

\begin{lemma}[Adapted from lemma 3 in \cite{gamma}]\label{lemma2}
For any $\MDP = ( \stateSpace, \actionSpace, \transition, \learnedR, \learnedGamma )$ with $\learnedR$ bounded by $[0, R_{max}]$,
\begin{equation}
    \left\lVert\VFunction{\optimalPolicy{\gtR}{\learnedGamma}}{\gtR}{\learnedGamma} - \VFunction{\optimalPolicy{\learnedR}{\learnedGamma}}{\gtR}{\learnedGamma}\right\rVert_{\infty} \leq 2 \underset{\pi \in \Pi_{\learnedGamma}}{\max} \left\lVert\VFunction{\pi}{\gtR}{\learnedGamma} - \VFunction{\pi}{\learnedR}{\learnedGamma}\right\rVert_{\infty}
\end{equation}
\end{lemma}

\begin{proof}
$\forall s \in \stateSpace,$
\begin{align}
    &\VFunction{\optimalPolicy{\gtR}{\learnedGamma}}{\gtR}{\learnedGamma}(s) - 
    \VFunction{\optimalPolicy{\learnedR}{\learnedGamma}}{\gtR}{\learnedGamma}(s) \nonumber\\
    = & (\VFunction{\optimalPolicy{\gtR}{\learnedGamma}}{\gtR}{\learnedGamma}(s) - \VFunction{\optimalPolicy{\gtR}{\learnedGamma}}{\learnedR}{\learnedGamma}(s) )
    - (\VFunction{\optimalPolicy{\learnedR}{\learnedGamma}}{\gtR}{\learnedGamma}(s) - \VFunction{\optimalPolicy{\learnedR}{\learnedGamma}}{\learnedR}{\learnedGamma}(s) )
    + (\VFunction{\optimalPolicy{\gtR}{\learnedGamma}}{\learnedR}{\learnedGamma}(s) - \VFunction{\optimalPolicy{\learnedR}{\learnedGamma}}{\learnedR}{\learnedGamma}(s) ) \nonumber \\
    \leq & (\VFunction{\optimalPolicy{\gtR}{\learnedGamma}}{\gtR}{\learnedGamma}(s) - \VFunction{\optimalPolicy{\gtR}{ \learnedGamma}}{\learnedR}{\learnedGamma}(s)) 
    - (\VFunction{\optimalPolicy{\learnedR}{\learnedGamma}}{\gtR}{\learnedGamma}(s) - \VFunction{\optimalPolicy{\learnedR}{\learnedGamma}}{\learnedR}{\learnedGamma}(s))\nonumber\\
    \leq &2\underset{\pi \in \Pi_{\learnedGamma}}{\max} \lvert\VFunction{\pi}{\gtR}{\learnedGamma}(s) - \VFunction{\pi}{\learnedR}{\learnedGamma}(s)\rvert.
\end{align}
\end{proof}

Intuitively, Lemma \ref{lemma2} measures the difference between V-functions of two policies: the optimal policy under the ground-truth reward function $\gtR$, and the one under the estimated reward function $\learnedR$. Both are evaluated using $\gtR$ and the effective discount factor $\learnedGamma$. Lemma \ref{lemma2} shows that this difference is at most double the highest V-function difference among policies in $\Pi_{\learnedGamma}$, when evaluated with $\gtR$ and $\learnedR$. This connects the value difference between two optimal policies evaluated under the ground truth reward to the difference in the same policy evaluated using the ground truth and estimated rewards respectively.

\begin{lemma} \label{lemma3}
Let $\partialMDP = ( \stateSpace, \actionSpace, \transition )$ be a partial MDP, and let $\gtR \in \mathbb{R}^{\stateSpace \times \actionSpace}_{\geq 0}$ be the ground-truth reward function and $\learnedR \in \mathbb{R}^{\stateSpace \times \actionSpace}_{\geq 0}$ be the estimated reward function with the discount factor $\learnedGamma$, and let $\pi$ be a policy. Then the following inequality holds:
\begin{equation}
    \left\lVert\VFunction{\pi}{\gtR}{\learnedGamma} - \VFunction{\pi}{\learnedR}{\learnedGamma}\right\rVert_{\infty} \leq \frac{1}{1-\learnedGamma} \left\lVert\gtR - \learnedR\right\rVert_{\infty}
\end{equation}
\end{lemma}

\begin{proof}
From the bellman equation, we have:
\begin{equation}
\VFunction{\pi}{\gtR}{\learnedGamma} = \pi \gtR + \learnedGamma \pi \transition \VFunction{\pi}{\gtR}{\learnedGamma}.
\end{equation}
Rearrange this, we will have $\VFunction{\pi}{\gtR}{\learnedGamma} = (I_{\stateSpace} - \learnedGamma \pi \transition)^{-1}\pi \gtR$. Therefore, 
\begin{align}
   \VFunction{\pi}{\gtR}{\learnedGamma} - \VFunction{\pi}{\learnedR}{\learnedGamma} & = (I_{\stateSpace} - \learnedGamma \pi \transition)^{-1}\pi \gtR - (I_{\stateSpace} - \learnedGamma \pi \transition)^{-1}\pi \learnedR \nonumber \\
    & = (I_{\stateSpace} - \learnedGamma \pi \transition)^{-1}\pi (\gtR - \learnedR)
\end{align}
For the $L_{\infty}$ inequality, we simply observe:

\begin{align}
    \left\lVert\VFunction{\pi}{\gtR}{\learnedGamma} - \VFunction{\pi}{\learnedR}{\learnedGamma}\right\rVert_{\infty} & = \left\lVert(I_{\stateSpace} - \learnedGamma \pi \transition)^{-1}\pi (\gtR - \learnedR)\right\rVert_{\infty} \nonumber \\
    & \leq \left\lVert(I_{\stateSpace} - \learnedGamma \pi \transition)^{-1}\right\rVert_{\infty} \left\lVert\pi\right\rVert_{\infty} \left\lVert\gtR - \learnedR\right\rVert_{\infty} \nonumber \\
    & \leq \frac{1}{1-\learnedGamma}\left\lVert\gtR - \learnedR\right\rVert_{\infty}
\end{align}
where we exploited the fact that $\left\lVert(I_{\stateSpace} - \learnedGamma \pi \transition)^{-1}\right\rVert_{\infty} = \frac{1}{1-\learnedGamma}$ and that $\left\lVert\pi\right\rVert_{\infty} \leq 1$.
\end{proof}

Using these three lemmas, we can now proceed to prove Theorem \ref{theorem:value_bound}.
\begin{proof}
\textbf{Proof of Theorem \ref{theorem:value_bound}}
$\forall s \in \stateSpace$
\begin{equation}
    \VFunction{\optimalPolicy{\gtR}{\gtGamma}}{\gtR}{\gtGamma}(s) -\VFunction{\optimalPolicy{\learnedR}{\learnedGamma}}{\gtR}{ \gtGamma}(s) = (\VFunction{\optimalPolicy{\gtR}{\gtGamma}}{\gtR}{\gtGamma}(s) - \VFunction{\optimalPolicy{\gtR}{\gtGamma}}{\gtR}{ \learnedGamma}(s))  + ( \VFunction{\optimalPolicy{\gtR}{\gtGamma}}{\gtR}{ \learnedGamma}(s) - \VFunction{\optimalPolicy{\learnedR}{\learnedGamma}}{\gtR}{\gtGamma}(s)). 
\end{equation}

Using Lemma \ref{lemma1}, we obtain:
\begin{equation}
    \left\lVert\VFunction{\optimalPolicy{\gtR}{\gtGamma}}{\gtR}{\gtGamma} - \VFunction{\optimalPolicy{\gtR}{\gtGamma}}{\gtR}{ \learnedGamma}\right\rVert_{\infty} \leq \frac{\gtGamma - \learnedGamma}{(1-\gtGamma)(1-\learnedGamma)}R_{max}.
\end{equation}

For the second term, we derive the following:
\begin{align}
     \VFunction{\optimalPolicy{\gtR}{\gtGamma}}{\gtR}{\learnedGamma}(s) - \VFunction{\optimalPolicy{\learnedR}{\learnedGamma}}{\gtR}{ \gtGamma}(s) \leq & \VFunction{\optimalPolicy{\gtR}{ \gtGamma}}{\gtR}{\learnedGamma}(s) - \VFunction{\optimalPolicy{\learnedR}{ \learnedGamma}}{\gtR}{\learnedGamma}(s) \nonumber\\
    \leq &  \VFunction{\optimalPolicy{\gtR}{\learnedGamma}}{\gtR}{\learnedGamma}(s) - \VFunction{\optimalPolicy{\learnedR}{\learnedGamma}}{\gtR}{\learnedGamma}(s) \nonumber\\
     \leq &  2\underset{\pi \in \Pi_{\learnedGamma}}{\max} \lvert\VFunction{\pi}{\gtR}{\learnedGamma}(s) - \VFunction{\pi}{\learnedR}{\learnedGamma}(s)\rvert. \label{equ:V_bound}
\end{align}
The first line results from $\learnedGamma \leq \gtGamma$, the second line from $\VFunction{\optimalPolicy{\gtR}{ \learnedGamma}}{\gtR}{\learnedGamma}(s) \geq \VFunction{\optimalPolicy{\gtR}{ \gamma'}}{\gtR}{\learnedGamma}(s)$ for any $\gamma' \in [0, 1]$, and the last line utilizes Lemma \ref{lemma2}.

Using Lemma \ref{lemma3}, for any policy $\pi$, we get:
\begin{align} \label{equ:lemma3}
    \left\lVert \VFunction{\pi}{\gtR}{\learnedGamma} - \VFunction{\pi}{\learnedR}{\learnedGamma}\right\rVert_{\infty} \leq \frac{1}{1-\learnedGamma} \left\lVert \gtR - \learnedR\right\rVert_{\infty}
\end{align}

By combining Equations \ref{equ:V_bound} and \ref{equ:lemma3}, we arrive at the inequality:
\begin{equation}
   \left\lVert  \VFunction{\optimalPolicy{\gtR}{\gtGamma}}{\gtR}{\learnedGamma} - \VFunction{\optimalPolicy{\learnedR}{\learnedGamma}}{\gtR}{ \gtGamma}\right\rVert_{\infty}\leq \frac{2}{1-\learnedGamma} \left\lVert \gtR - \learnedR\right\rVert_{\infty}
\end{equation}
\end{proof}
\subsection{Expert Policy Estimation Error Bound}

In this section, we establish a bound on the expert policy estimation error based on the coverage of expert demonstrations. Specifically, we use $\expertFilterC{\expertPolicy}\expertFilter{\approxExpertPolicy}\zeta$ to measure the expert policy estimation error and constrain the reward function estimation error:
\begin{equation}
    \left|\gtR - \learnedR\right| \leq \expertFilterC{\expertPolicy}\expertFilter{\approxExpertPolicy}\zeta.
\end{equation}
We bound this error term $\expertFilterC{\expertPolicy}\expertFilter{\approxExpertPolicy}\zeta$ by the number of expert-demonstrated state-action pairs. The derivation involves three key steps:

\begin{enumerate}
    \item State the strategy for estimating the expert policy $\approxExpertPolicy$ from $N$ samples of expert-demonstrated state-action pairs,
    \item Compute the expectation of the estimated expert policy $\approxExpertPolicy$ based on $N$ samples and subsequently the expectation of the expert policy estimation error $\expertFilterC{\expertPolicy}\expertFilter{\approxExpertPolicy}\zeta$.
    \item Apply McDiarmid's Inequality to $\expertFilterC{\expertPolicy}\expertFilter{\approxExpertPolicy}\zeta$ as a function of $N$ expert-demonstrated samples to bound the expert policy estimation error.
\end{enumerate}

After simplification, the bound on the expert policy estimation error, in terms of the number of expert-demonstrated samples $N$ and the effective $\learnedGamma$, is given by:

\begin{equation}    \Pr\left(\expertFilterC{\expertPolicy}\expertFilter{\approxExpertPolicy}\zeta(s, a)\geq t\right) \leq \exp\left( \frac{-2t^2 (1 - \gamma)^2}{N R_{\max}^2} \right)
\end{equation}

\subsubsection{Expert Policy Estimation Strategy}
\label{subsubsection:expert_policy_estimation_strategy}
\paragraph{Setup} We estimate the expert policy $\approxExpertPolicy$ from $N$ independently sampled state-action pairs from the expert policy $\expertPolicy$. For simplicity, we assume a discrete state and action space. Thus, the expert policies can be represented as a matrix of size $|S| \times |A|$, where rows correspond to states and columns to actions. Given that $\expertPolicy$ is deterministic, each row of $\approxExpertPolicy$ contains at most one entry with a value of 1, indicating that action $a$ is chosen at state $s$; all other entries are 0. Formally, the estimated $\expertPolicy$ can be represented as:
\[
\expertPolicy = \begin{pmatrix} 
x_{1,1} & x_{1,2} & \cdots & x_{1,A} \\ 
x_{2,1} & x_{2,2} & \cdots & x_{2,A} \\ 
\vdots & \vdots & \ddots & \vdots \\ 
x_{S,1} & x_{S,2} & \cdots & x_{S,A} 
\end{pmatrix}
\]
where $x_{s,a} \in \{0, 1\}$, and for each $s \in \{1, \dots, |S|\}$, exactly one entry in the row is 1, while the rest are 0. For each demonstrated state-action pair, we represent them as a matrix of the same size $\widehat{\pi}^E_i \in \{0, 1\}^{|S| \times |A|}$, and there is one row $s_i$ such that:
\[
\widehat{\pi}^E_i(s_i, :) = \expertPolicy(s_i, :)
\]
That is, for each matrix $\widehat{\pi}^E_i$, there exists a row $s_i$ where $\widehat{\pi}^E_i(s_i, :) = \expertPolicy(s_i, :)$. The remaining rows of $\widehat{\pi}^E_i$ are not necessarily useful and contain value 0. We want to devise an estimation strategy that maps the $N$ observed matrices $\widehat{\pi}^E_1, \widehat{\pi}^E_2, \dots,\widehat{\pi}^E_N)$ to the estimated matrix $\approxExpertPolicy$.

\paragraph{Estimation} We describe the estimation strategy below. We first take the element-wise sum for all observed partial expert policy matrices:
\begin{equation}
  \widehat{\Pi}^E_N(s, a)   = \sum_{i=1}^N \widehat{\pi}^E_i(s, a)
\end{equation}
Specifically, each of the entry of $\widehat{\Pi}^E_N(s, a)$ counts how many times action $a$ was chosen for state $s$ across the $N$ observed matrices. We estimate $\approxExpertPolicy$ from the aggregated $\widehat{\Pi}^E_N$. 

For each row $s$, estimate the action $a$ that corresponds to the column where the value is 1 by finding the column with the uniquely maximum count:
\[
\hat{a}_s = \arg\max_{a} \widehat{\Pi}^E_N(s, a) \text{\;such that
\;} \widehat{\Pi}^E_N(s, \hat{a}_s) > \widehat{\Pi}^E_N(s, a), \forall a \neq \hat{a}_s
\]
Thus, the estimated row $\approxExpertPolicy(s, :)$ is given by placing a 1 in column $\hat{a}_s$ and 0 in all other columns:
\[
\approxExpertPolicy(s, a) = \begin{cases} 
1, & \text{if } a = \hat{a}_s \\ 
0, & \text{otherwise} 
\end{cases}
\]
For rows that do not have a single action with the highest count, set all entries in the row to 0. Thus the final $\approxExpertPolicy$ is estimated by choosing the column with the highest vote from the observed matrices.

\subsubsection{Calculating the Expected Value}
\label{subsubsection:expected_value_calculation}
Given that there are \(|S|\) states and each of the $N$ demonstrated states is uniformly and independently sampled, we first compute of the expectation of the estimated policy \(\mathbb{E}[\approxExpertPolicy]\), then the expectation of the expert policy estimation error $\expertFilterC{\expertPolicy}\expertFilter{\approxExpertPolicy}\zeta$.

\paragraph{Expected Value of Estimated Policy}The count for each action in the aggregated matrix \(\widehat{\Pi}^E_N(s, a)\) depends on how often state \(s\) is sampled in \(N\) independent demonstrations. Since the states are sampled uniformly, the probability of observing state \(s\) in any single demonstration is \(\frac{1}{|S|}\). Let \(X_s\) be the random variable representing the number of times state \(s\) is sampled in \(N\) independent demonstrations. \(X_s\) follows a binomial distribution: $X_s \sim \text{Binomial}(N, \frac{1}{|S|})$. Thus, the expected number of times state \(s\) is sampled across the \(N\) demonstrations is:
\[
\mathbb{E}[X_s] = N \cdot \frac{1}{|S|}.
\]

For each time state \(s\) is sampled, the expert chooses the correct action \(a_s^*\). Therefore, the expected count for action \(a_s^*\) at state \(s\) is: 

\[\mathbb{E}[\widehat{\Pi}^E_N(s, a_s^*)] = \mathbb{E}[X_s] = N \cdot \frac{1}{|S|}.\]

For any action \(a \neq a_s^*\), the expert never chooses it, so the expected count for those actions is $0$:

\[
\mathbb{E}[\widehat{\Pi}^E_N(s, a)] = 0 \quad \text{for all } a \neq a_s^*.
\]

The probability that \(a_s^*\) is the unique maximum can be approximated by the probability that state \(s\) is sampled at least once. This probability is given by
$1 - \left(1 - \frac{1}{|S|}\right)^N$. Therefore, the probability that the correct action \(a_s^*\) is the unique maximum for state \(s\) is approximately:

\[
P_s^{\text{unique}} \approx P_s^{\text{sampled}} = 1 - \left(1 - \frac{1}{|S|}\right)^N
\]

The expectation of the estimated policy \(\mathbb{E}[\approxExpertPolicy]\) for each state \(s\) is given by:

\[
\mathbb{E}[\approxExpertPolicy(s, :)] = P_s^{\text{unique}} \cdot \expertPolicy(s, :)
\]

Substituting \(P_s^{\text{unique}}\) with the expression derived above:

\[
\mathbb{E}[\approxExpertPolicy(s, a)] =
\begin{cases} 
1 - \left(1 - \frac{1}{|S|}\right)^N & \text{for } a = a^*_s,\\
0 & \text{otherwise}.
\end{cases}
\]

The expectation of the estimated expert policy \(\mathbb{E}[\approxExpertPolicy]\) is scaled by the probability that each state \(s\) is sampled at least once across the \(N\) demonstrations. This probability increases as \(N\) grows, and in the limit as \(N \to \infty\), \(\mathbb{E}[\approxExpertPolicy]\) converges to the true expert policy \(\expertPolicy\). 

\paragraph{Expected Value of Expert Policy Estimation Error}

We compute the expected value of the operation \(\expertFilterC{\expertPolicy}\expertFilter{\approxExpertPolicy}\zeta\) given \(N\) demonstrated expert state-action pairs. 

We recap the key definitions for the ease of reading:
\begin{itemize}
    \item $\expertFilter{\approxExpertPolicy}(.)$ represents the \textit{expert-filter} that preserves function values only for actions taken by the approximate expert policy $\approxExpertPolicy(a|s)$,
    \item $\expertFilterC{\expertPolicy}(.)$ is the \textit{expert-filter-complement} that preserves values for actions not taken by the ground-truth expert policy $\expertPolicy(a|s)$,
    \item $\zeta$ is a non-negative function bounded by \(\frac{R_{\max}}{1-\gamma}\).
\end{itemize}

We are interested in computing the expectation of the operation $\mathbb{E}[\expertFilterC{\expertPolicy}\expertFilter{\approxExpertPolicy}\zeta]$. This involves two steps:

In step 1, the expert filter \(\expertFilter{\approxExpertPolicy}\zeta\) keeps the values of \(\zeta(s, a)\) for actions \(a\) that are taken by the estimated expert policy \(\approxExpertPolicy(a|s)\). From the expectation of \(\approxExpertPolicy\), we know that:
\[
\mathbb{E}[\approxExpertPolicy(s, a)] = P_s^{\text{unique}} \cdot \mathds{1}\{ a = a_s^* \}
\]
Thus, the expected value of \(\expertFilter{\approxExpertPolicy}\zeta(s, a)\) is:
\begin{align}
\mathbb{E}[\expertFilter{\approxExpertPolicy}\zeta(s, a)] & = \zeta(s, a) \cdot \mathbb{E}[\mathds{1}\{\approxExpertPolicy(a|s) > 0\}]  \\
& = P_s^{\text{unique}} \cdot \zeta(s, a_s^*) \cdot \mathds{1}\{ a = a_s^*\}  
\end{align}
For any other action \(a \neq a_s^*\), \(\mathbb{E}[\expertFilter{\approxExpertPolicy}\zeta(s, a)] = 0\), since \(\approxExpertPolicy(s, a)\) will not take that action.

In step 2, we apply the expert-filter complement \(\expertFilterC{\expertPolicy}\), which preserves the function values for actions not taken by the true expert policy \(\expertPolicy\). The true expert policy \(\expertPolicy\) deterministically selects action \(a_s^*\) for each state \(s\), so the expert-filter complement will only retain the values for actions \(a \neq a_s^*\).

Thus, after applying the expert-filter complement, the expected value becomes:
\[
\mathbb{E}[\expertFilterC{\expertPolicy} \expertFilter{\approxExpertPolicy}\zeta(s, a)] = \mathbb{E}[\expertFilter{\approxExpertPolicy}\zeta(s, a)] \cdot \mathds{1}\{ \expertPolicy(a|s) = 0 \}
\]

Since \(\expertPolicy(a|s) = 0\) for all \(a \neq a_s^*\), we only preserve values for actions other than \(a_s^*\). Therefore, we have:

\[
\mathbb{E}[\expertFilterC{\expertPolicy}\expertFilter{\approxExpertPolicy}\zeta(s, a)] = 
\begin{cases}
P_s^{\text{unique}} \cdot \zeta(s, a_s^*) \cdot \mathds{1}\{ a = a_s^*\} & \text{if } a \neq a_s^* \\
0 & \text{otherwise}
\end{cases}
\]

We now combine the two steps to compute the full expectation \(\mathbb{E}[\expertFilterC{\expertPolicy}\expertFilter{\approxExpertPolicy}\zeta]\). Since the expert-filter complement \(\expertFilterC{\expertPolicy}\) retains only the values for actions \(a \neq a_s^*\), and \(\expertFilter{\approxExpertPolicy}\) retains only the values for action \(a_s^*\), the result of applying both filters will be 0 for all \(s\) and \(a\), because no action can satisfy both conditions simultaneously.

Thus, the expected value of the operation is:

\[
\mathbb{E}[\expertFilterC{\expertPolicy}\expertFilter{\approxExpertPolicy}\zeta(s, a)] = 0
\]

\subsubsection{Applying McDiarmid's Inequality on the Expert Policy Estimation Error}
\label{subsubsection:applying_mcdiarmid_inequality}
We apply McDiarmid's Inequality to bound the error in the expert policy estimation, specifically \(\mathbb{E}[\expertFilterC{\expertPolicy}\expertFilter{\approxExpertPolicy}\zeta]\), where \(\zeta\) is bounded by \(\frac{R_{\max}}{1-\gamma}\).

\paragraph{McDiarmid's Inequality}
Let \(X_1, X_2, \dots, X_N\) be independent random variables taking values in a space \(\mathcal{X}\), and let \(f: \mathcal{X}^N \to \mathbb{R}\) be a function that satisfies a bounded difference property. Specifically, if there exist constants \(c_i\) such that:

\[
\sup_{x_1, \dots, x_N, x_i'} \left| f(x_1, \dots, x_i, \dots, x_N) - f(x_1, \dots, x_i', \dots, x_N) \right| \leq c_i
\]

then McDiarmid’s Inequality states that:

\[
\Pr\left( f(X_1, \dots, X_N) - \mathbb{E}[f(X_1, \dots, X_N)] \geq t \right) \leq \exp\left( \frac{-2 t ^2}{\sum_{i=1}^N c_i^2} \right)
\]

\paragraph{The Function \(f(\cdot)\) in Our Case}

Let the function \(f(\widehat{\pi}^E_1, \widehat{\pi}^E_2, \dots, \widehat{\pi}^E_N)\) represent \(\expertFilterC{\expertPolicy}\expertFilter{\approxExpertPolicy}\zeta\), which is determined by \(N\) independent state-action pairs sampled from the expert policy. Based on the boundedness of \(\zeta\) by \(\frac{R_{\max}}{1-\gamma}\), we want to bound the deviation of this function from its expectation.

\paragraph{Sensitivity (Bounded Difference Property)}

For each sample $\widehat{\pi}^E_i$, changing it affects at most one row of the estimated policy \(\approxExpertPolicy(s, :)\), and consequently, the value of the function \(\expertFilterC{\expertPolicy}\expertFilter{\approxExpertPolicy}\zeta\) is only affected in that row. The magnitude of this change is limited by the bound on \(\zeta\), which is \(\frac{R_{\max}}{1 - \gamma}\).

Thus, for any change in a single sample \(\widehat{\pi}^E_i\), the change in the value of the function is bounded by:

\[
\left| f(\widehat{\pi}^E_1, \dots, \widehat{\pi}^E_i, \dots, \widehat{\pi}^E_N) - f(\widehat{\pi}^E_1, \dots, (\widehat{\pi}^{E}_i)', \dots, \widehat{\pi}^E_N) \right| \leq \frac{R_{\max}}{1 - \gamma}
\]

This gives us the sensitivity constant \(c_i = \frac{R_{\max}}{1 - \gamma}\) for each sample \(\widehat{\pi}^E_i\).

\paragraph{Applying McDiarmid's Inequality}

Now, we apply McDiarmid's Inequality with \(c_i = \frac{R_{\max}}{1 - \gamma}\) for all \(i = 1, 2, \dots, N\). The inequality becomes:

\[
\Pr\left(\expertFilterC{\expertPolicy}\expertFilter{\approxExpertPolicy}\zeta - \mathbb{E}[\expertFilterC{\expertPolicy}\expertFilter{\approxExpertPolicy}\zeta] \geq t \right) \leq \exp\left( - \frac{2t^2}{N \left( \frac{R_{\max}}{1 - \gamma} \right)^2} \right)
\]

Simplifying the bound, the probability that the expert policy estimation error, measured by \(\expertFilterC{\expertPolicy}\expertFilter{\approxExpertPolicy}\zeta\), deviates from its expected value by more than \(\epsilon\) is bounded by:

\begin{equation}  \Pr\left(\expertFilterC{\expertPolicy}\expertFilter{\approxExpertPolicy}\zeta(s, a) \geq t\right) \leq \exp\left( \frac{-2t^2 (1 - \gamma)^2}{N R_{\max}^2} \right) \label{equ:mcdiarmid_inequality}
\end{equation}

As $N$ increases, the probability of significant deviations from the expert policy estimation error diminishes exponentially, providing a high-confidence guarantee that the estimated policy will converge to the true expert policy. 

Given the state space \(S\) and policy space \(\Pi_{\hat{\gamma}}\) for the effective horizon \(\hat{\gamma}\), we want the bound to hold uniformly over all \((s, \pi)\) pairs. Using the union bound, we ensure the total failure probability remains below \(\frac{\delta}{|S| |\Pi_{\hat{\gamma}}|}\), hence each \((s, \pi)\) pair has a small enough failure probability so that, collectively, the probability of any pair violating the bound is at most \(\delta\). By setting the RHS of the bound to \(\frac{\delta}{|S| |\Pi_{\hat{\gamma}}|}\), we solve for \(t\):
\begin{align}
    \exp\left( \frac{-2t^2 (1 - \gamma)^2}{N R_{\max}^2} \right) &= \frac{\delta}{\lvert \stateSpace \rvert \lvert \Pi_{\learnedGamma} \rvert} \nonumber \\
    \frac{2t^2 (1 - \gamma)^2}{N R_{\max}^2} &= -\ln\left( \frac{\delta}{\lvert \stateSpace \rvert \lvert \Pi_{\learnedGamma} \rvert} \right) \nonumber \\
    t^2 &= \frac{N R_{\max}^2}{2 (1 - \gamma)^2} \ln\left( \frac{\lvert \stateSpace \rvert \lvert \Pi_{\learnedGamma} \rvert}{\delta} \right) \nonumber \\
    t &= \frac{R_{\max}}{1-\learnedGamma} \sqrt{ \frac{N}{2} \ln\left( \frac{\lvert \stateSpace \rvert \lvert \Pi_{\learnedGamma} \rvert}{\delta} \right) }
\end{align}
We obtain the threshold that bounds the expert policy estimation error with probability at least $1-\delta$, uniformly across all state-policy pairs.

At first glance, it may seem counterintuitive that the error bound threshold \(t\) increases with the number of samples \(N\), as indicated in the expression of $t$. However, this bound should be interpreted in terms of the probability of the estimation error exceeding \(t\), i.e., \(\Pr(\text{estimation error} > t)\). As \(N\) increases, the threshold \(t\) grows, but the probability of the error exceeding this larger bound decreases exponentially. 

This means that, though larger deviations are allowed with more samples, these deviations become increasingly unlikely as \(N\) increases. Therefore, for a fixed level of confidence (i.e., fixed \(\delta\)), the overall estimation error \textit{actually decreases} with more data, as the model becomes more accurate, and the likelihood of significant errors diminishes. The increasing threshold merely reflects the fact that we are accounting for possible rare deviations, but these deviations are becoming less probable as \(N\) grows.

\paragraph{Interpreting the Error Bound}
We rearrange the original bound to focus on how the estimation error behaves as $N$ grows, for a fixed probability $\delta$. We can rewrite the bound to emphasize the dependence on \( N \) and \( \delta \) more explicitly:

\[
t = \frac{R_{\max}}{1 - \gamma} \sqrt{ \frac{N}{2} \ln\left( \frac{\lvert \stateSpace \rvert \lvert \Pi_{\learnedGamma} \rvert}{\delta} \right) }
\]

We have the following key observations:

\textbf{Dependence on \( N \):} The term \( \sqrt{N} \) shows that the error bound increases as \( N \) increases, but the rate of increase is proportional to the square root of \( N \). This means that to significantly reduce the error, you need a large increase in \( N \), as the error decreases inversely with \( \sqrt{N} \).

\textbf{Dependence on \( \delta \) (confidence level):} The logarithmic term \( \ln\left( \frac{\lvert S \rvert \lvert \Pi_{\learnedGamma} \rvert}{\delta} \right) \) captures the effect of the confidence level. As \( \delta \) (the failure probability) decreases, the logarithmic term grows, making \( t \) larger. This means that requiring higher confidence (lower \( \delta \)) results in a higher error bound, reflecting the trade-off between confidence and error tolerance.

\textbf{Simplified Bound to Show Error Decrease:}

For a fixed confidence level (i.e., for fixed \( \delta \) and \( \lvert \stateSpace \rvert \lvert \Pi_{\learnedGamma} \rvert \)), you can express the bound more clearly as:

\[
t = O\left( \frac{R_{\max}}{1 - \gamma} \sqrt{N} \right)
\]
This shows that, for fixed confidence, the error bound increases with the square root of \( N \), and as a result, the probability that the error exceeds this bound decreases. In practice, this means that while larger sample sizes allow for larger deviations, these deviations become increasingly unlikely as \( N \) increases. By writing the bound in this direct form, we make it clear that as \( N \) increases, the error bound becomes more controlled (grows slowly compared to \( N \)) while the probability of exceeding this bound shrinks exponentially. This formulation provides an intuitive and direct interpretation of the trade-off between error, sample size, and confidence level.

\section{Details of the Tasks} 
\label{appendix:task_details}
All tasks share a common action space, transition function with $9$ actions per state, and a $10\%$ chance of moving randomly.

\textbf{Gridworld-simple}. In each Gridworld-simple environment, there are $10 \times 10$ states with $4$ randomly selected goal states. The ground-truth reward function assigns $+1$ for goal states and $0$ for all others. The expert policy moves greedily towards the nearest goal.

\textbf{Gridworld-hard}. The Gridworld-hard environment consists of $15 \times 15$ states and $6$ random goal states. It has a larger state space and more goal states than Gridworld-simple, but shares the same reward function.

\textbf{Objectworld-linear}. Each Objectworld-linear environment has $10 \times 10$ states and $10$ randomly placed objects, each with a randomly assigned inner and outer color from the set $C = \{c_0, c_1\}$. Object states receive a reward of $+3$ for outer color $c_0$ and $+1$ for outer color $c_1$. All other states have a $0$ reward.

\textbf{Objectworld-nonlinear}. Objectworld-nonlinear differs from its linear counterpart only in the reward function. The reward depends on the outer colors of surrounding objects: $+1$ for states within 3 grids of $c_0$ and 2 grids of $c_1$, $-1$ for states within 3 grids of $c_1$, and $0$ otherwise.
\section{Modifications to the Linear Programming IRL method (LP-IRL)}
\label{appendix:LP-IRL_objective}
The original objective function of LP-IRL \citep{AndrewNg} is formulated on the full state coverage of the expert policy and the ground-truth $\gtGamma$. We modified the objective function to extend it to our setting with varying partial state coverage and $\learnedGamma$. W.L.O.G, let $a_1$ be the uniquely optimal action for all states. The transition probability vector at state $s$ for the expert action is thus $\Pe(s) = P(.|s, a_1)$, while that for the rest of the actions $a \in \{a_2, ..., a_k\}$ be $\Pa(s) = P(.|s, a)$. Our modified objective function for LP-IRL is as follows:

Given expert demonstrations covering $N$ states,
\begin{align}
   \underset{\learnedR}{\max}\sum_{i=1}^N \underset{a \in \{a_2, ..., a_k\}}{\min} &(\Pe(i)-\Pa(i))(I- \learnedGamma \approPe)^{-1} \learnedR  \nonumber \\
   \text{subject to\;\;} & (\Pe(i)-\Pa(i))(I- \learnedGamma \approPe)^{-1}\learnedR \geq b \nonumber \\
   & b>0, \forall a  \in \{a_2, ..., a_k\} \nonumber\\
   & |\learnedR| \leq R_{max}, \text{for\;} i = 1, ..., N,
   \label{equ:lp_objective}
\end{align}
whereas $\Pe$ is full expert policy transition matrix of size $|\stateSpace| \times |\stateSpace|$ such that $[\Pe](s, s') = \transition(s'| a_1, s)$, while $\approPe$ is an estimate of $\Pe$ from partial expert demonstrations covering only $N \leq |\stateSpace|$ states. Essentially, the transition matrix of the expert policy $\Pe$ is the combination of the transition function $\transition$ (known) and the expert policy $\pi^E$ (unknown). Therefore, estimating transition matrix of the expert policy $\approPe$ is equivalent to estimating the expert policy from expert demonstrations. Our modifications are highlighted below.

\begin{algorithm}[H]
\caption{LP-IRL with Partial Expert Demonstrations}\label{alg:lpirl}
\begin{algorithmic}[1]
    \Require State space $\mathcal{S}$; Action space $\mathcal{A}$;
    $P\in\mathbb{R}^{|\mathcal{S}|\times |\mathcal{A}|\times |\mathcal{S}|}$; expert policy $\pi_E:\{1,\ldots,|\mathcal{S}|\}\to\{1,\ldots,|\mathcal{A}|\}$; 
    demonstration set $\mathcal{D}\subseteq\{1,\ldots,|\mathcal{S}|\}$; discount factor $\hat{\gamma}\in(0,1)$; 
    maximum reward $R_{\max}>0$; margin $\beta>0$.
    
    \Ensure Reward vector $R\in\mathbb{R}^n$.
    
    \State \textbf{Estimate Expert Transitions:}
    \For{$s=1,\ldots,|\mathcal{S}|$}
        \If{$s\in\mathcal{D}$}
            \State $P_E(s,:) \gets P(s,\pi_E(s),:)$.
        \Else
            \State $P_E(s,:) \gets \mathbf{0}$.
        \EndIf
    \EndFor
    
    \State \textbf{Compute Mapping Matrix:}
    \For{each $s=1,\ldots,|\mathcal{S}|$, $a\in\{1,\ldots,|\mathcal{A}|\}\setminus\{\pi_E(s)\}$}
        \State $F^{\hat{\gamma}}(s,a) \gets \Bigl(P_E(s,:) - P(s,a,:)\Bigr) \left(I - \hat{\gamma}\,P_E\right)^{-1}$.
    \EndFor
    \State Set $F^{\hat{\gamma}}(s,a) \gets \mathbf{0}$ for all $s\notin\mathcal{S}$.
    
    \State \textbf{Formulate LP:}
    \State \quad \textbf{Objective:} Minimize 
    \[
    c^\top x,\quad \text{with}\quad c = -\begin{bmatrix}\mathbf{0}_n \\ \mathbf{1}_n \\ -\ell_1\,\mathbf{1}_n\end{bmatrix},\quad x = \begin{bmatrix}V \\ R \\ \xi\end{bmatrix}\in\mathbb{R}^{3n}.
    \]
    \State \quad \textbf{Constraints:} For each $s\in\{1,\dots,n\}$ and each $a\in A\setminus\{\pi(s)\}$, require
    \[
    F^{\hat{\gamma}}(s,\pi(s))\,R - F^{\hat{\gamma}}(s,a)\,R \ge \beta,
    \]
    and enforce the reward bounds
    \[
    -R_{\max}\le R(s)\le R_{\max},\quad \forall\, s.
    \]
    
    \State \textbf{Solve LP:} Assemble the constraints into the matrix inequality $D\,x\le b$ and solve 
    \[
    \min_{x\in\mathbb{R}^{3n}} \; c^\top x \quad \text{s.t.} \quad D\,x\le b.
    \]
    
    \State \Return Extract $R\in\mathbb{R}^n$ from the optimal solution $x^*$.
\end{algorithmic}
\end{algorithm}

\textbf{Estimate Expert Transition Matrix.} We estimate $\approPe$ by setting the transitions of the undemonstrated states to 0. Using this estimation, the action choices at the undemonstrated states do not affect the constraints on the other states. In addition, the value of the undemonstrated state myopically reduced to the state reward function. 

\textbf{Use Different Discount Factors.} Our objective function also allows $\learnedGamma$ to differ from the ground-truth $\gtGamma$. That is, with an estimated $\approPe$, we allow the mapping $\approFFunction{i}{\learnedGamma} = (\Pe(i)-\Pa(i))(I- \learnedGamma \approPe)^{-1}$ to take different $\learnedGamma$s to account the temporal effect of the policy at different horizons.

\textbf{Remove L1 Regularization.}
The original LP-IRL uses the $L_1$-regularization on the reward function, i.e., $|R|$, to resolve the ambiguity among the feasible reward functions. This assumption on the preferred reward function form can be a confounding factor to the performance of the induced policy, therefore we remove it to focus on the effect of horizons.

\textbf{Enforce Uniquely Optimal Expert Policy.} To enforce the induced policy to be uniquely optimal, we modified the constraints to be strictly positive, i.e., $b > 0$.
\section{Modifications to the Maximum Entropy IRL method (MaxEnt-IRL)}
\label{appendix:maxentirl}

Our theoretical analysis of the effective horizon stems from the error bound between two feasible reward function sets: the ground truth expert policy set and the approximated expert policy set derived from limited expert demonstrations. However, most IRL algorithms, such as Maximum Entropy IRL~\citep{maxentirl} and Maximum Margin IRL~\citep{maxmargin_IRL}, obtain a single reward function from the feasible set using specific criteria. The error bound between the learned reward function and the ground-truth $\gtR$ may deviate from the bound between their corresponding feasible sets. In this section, we extend our theoretical insights to Maximum Entropy Inverse Reinforcement Learning~\citep{maxentirl} and demonstrate that our conclusions on the effective horizon are still applicable when choosing the reward function based on specific criteria.

\begin{figure*}
    \centering
    \begin{tabular}{c}

\includegraphics[width=0.95\textwidth]{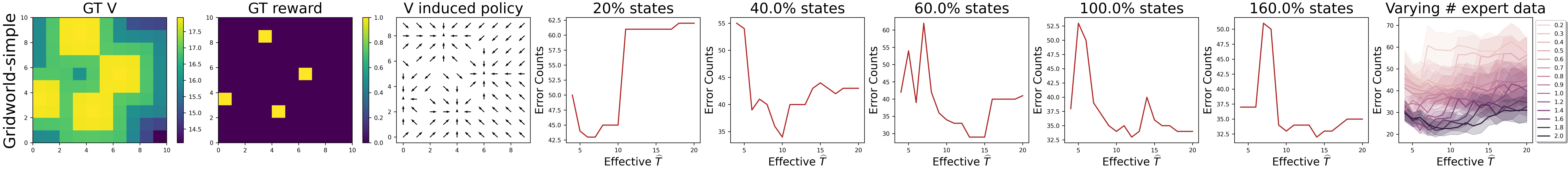}
         \\
         
      \includegraphics[width=0.95\textwidth]{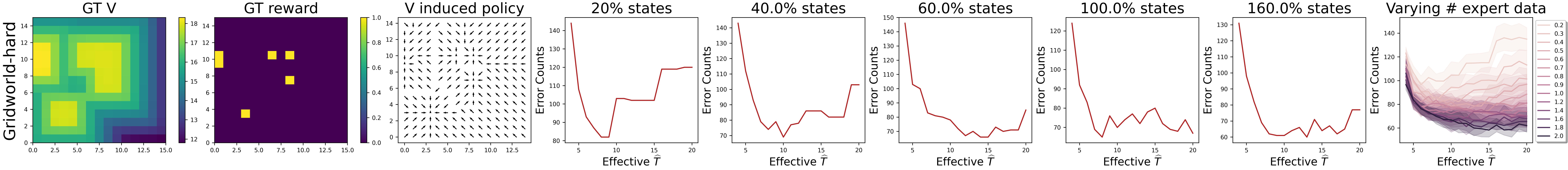}
        \\ 
        
        \includegraphics[width=0.95\textwidth]{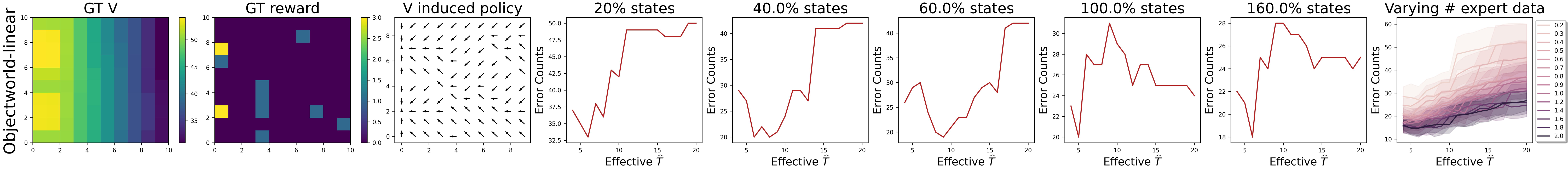}
        \\ 
    
        \includegraphics[width=0.95\textwidth]{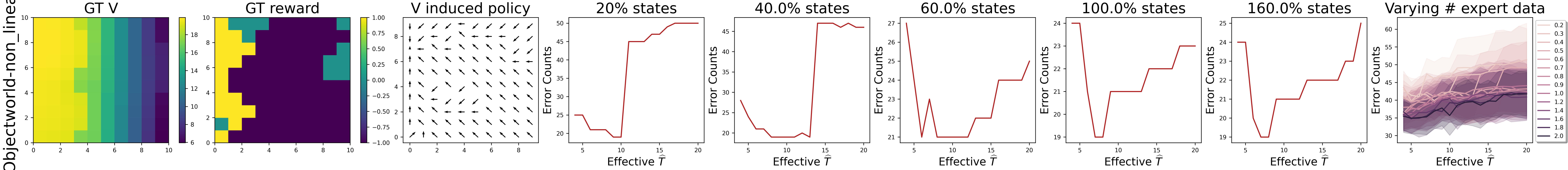}
    \end{tabular}
\caption{Summary of MaxEnt-IRL with different horizons for the four tasks. For each task, we present the ground-truth value function (column 1), ground-truth reward function (column 2), expert policy (column 3), error count curves in all states for different amount of expert data for a \textbf{single instance} of the task (columns 4-8), and finally, a summary of the error curves for a \textbf{batch} of 10 MDPs (column 9). In all four tasks, the ground-truth horizon $\gtHorizon = 20$. The optimal horizon $\optimalHorizon<\gtHorizon$ for varying amount of expert data.} 
\label{fig:maxent_envs}
\Description{MaxentIRL summary}
\end{figure*}

\subsection{Maximum Entropy IRL}
MaxEnt-IRL is formulated for finite horizon MDPs with horizon $\gtHorizon$ and no discounting. The reward function is linear in the state feature, i.e., $r = \theta^{\gtHorizon}\phi(s)$. The reward of a trajectory $\tau = (s_0, ..., s_{\gtHorizon})$ is the undiscounted sum of all state rewards, where $f_{\gtHorizon}$ is the total feature count for the trajectory $\tau$. MaxEnt IRL aims to match the feature counts between the policy and the expert demonstrations while maximizing the entropy of the induced policy. This leads to a distribution over behaviors constrained to match feature expectations, as shown in Equation (\ref{equ:maxent_dis}). 
\begin{equation}
    p(\tau|\theta, \transition, \gtHorizon) \approx \frac{1}{Z(\theta, \transition, \gtHorizon)} \exp^{\theta^T f_{\gtHorizon}} \prod_{s_t, a_t, s_{t+1} \in \tau} \transition(s_{t+1}|s_t, a_t) \label{equ:maxent_dis}
\end{equation}

The MaxEnt IRL objective is to choose the reward parameter $\theta$ that maximizes the probability of expert demonstrations under the distribution in Equation (\ref{equ:maxent_dis}):
\begin{equation}
    \theta^* := \underset{\theta}{\arg \max}\;\; L(\theta) = \underset{\theta}{\arg\max} \sum_{\tau_i \in demo} \log p(\tau_i|\theta, \transition, \gtHorizon) \label{equ:maxent_objective}
\end{equation}

The objective function in equation (\ref{equ:maxent_objective}) can be optimized using gradient-based methods, where the gradient is the difference between the empirical and learner's expected feature counts, expressed in terms of expected state visitation frequencies, $D_{s_i, \gtHorizon}$.
\begin{equation}
    \nabla L(\theta) = \tilde{f} - \sum_{\tau \in demo} p(\tau|\theta, \transition)f_{\gtHorizon} = \tilde{f} - \sum_{s_i}D_{s_i, \gtHorizon}f_{s_i} \label{equ:maxent_gradient}
\end{equation} 

\subsection{Maximum Entropy IRL with varying Effective Horizon}

To extend MaxEnt IRL to varying effective horizons, we optimize trajectories with horizon $\learnedHorizon \leq \gtHorizon$. The corresponding optimal trajectory distribution is given in Equation (\ref{equ:maxent_dis_discounted}), where $f_{\learnedHorizon}$ is the undiscounted sum of $\learnedHorizon$ consecutive states in trajectory $\tau = (s_0, ..., s_{\learnedHorizon})$.

\begin{equation}
    p(\tau|\theta, \transition, \learnedHorizon) \approx \frac{1}{Z(\theta, \transition, \learnedHorizon)} \exp^{\theta^T f_{\learnedHorizon}} \prod_{t = 0}^{\learnedHorizon-1} \transition(s_{t+1}|s_t, a_t) \label{equ:maxent_dis_discounted}
\end{equation}
whereas the feature factor $f_{\learnedHorizon} = \sum_{t=0}^{\learnedHorizon} \phi(s_t)$ is the undiscounted sum of $\learnedHorizon$ consecutive states in trajectory $\tau = (s_0, ..., s_{\learnedHorizon})$.

To investigate the performance of varying effective horizons under the same amount of expert data, we adjust the expert trajectories' lengths to match the planning horizons $\learnedHorizon$, while maintaining the total number of demonstrated states constant. Given that the expert covers $N$ states, for each effective horizon $\learnedHorizon$, we collect $N//\learnedHorizon$ trajectories with random initial states. Consequently, the gradient of the objective function is modified as shown in Equation (\ref{equ:maxent_gradient_discounted}): 
\begin{align}
    \nabla L(\theta, \learnedHorizon)  = \tilde{f}_{\learnedHorizon} - \sum_{\tau \in demo} p(\tau|\theta, \transition, \learnedHorizon)f_{\learnedHorizon}   = \tilde{f}_{\learnedHorizon} - \sum_{s_i}D_{s_i, \learnedHorizon}f_{s_i} \label{equ:maxent_gradient_discounted}
\end{align} 
where the expert expected feature vector $\tilde{f}_{\learnedHorizon}$ and the induced state visitation frequency $D{s_i, \learnedHorizon}$ consider only trajectories with length $\learnedHorizon$.

Our method for collecting expert demonstrations differs from the standard IRL approaches, where expert demonstrations are fixed and given upfront. Instead, we generate expert trajectories independently for each horizon $\learnedHorizon$ by rolling out trajectories with random initial states using the expert policy. 

An alternative approach is to augment the fixed set of expert trajectories with length $\gtHorizon$ by breaking them into $\gtHorizon-\learnedHorizon+1$ shorter segments each with length $\learnedHorizon$. However, this method can slightly underestimate the visitation frequency of the final states of the original trajectories due to uniform sampling over the first $\gtHorizon-\learnedHorizon+1$ states. Our data collection method eliminates the potential bias that may arise from the quality of the expert demonstrations. In addition, our initial state distribution $\rho_0$ is a uniform distribution over the entire state space, which ensures a good approximation of the true state visitation frequency of the expert policy.

\subsection{Results}

\begin{figure}[htp]
\centering
\includegraphics[width=1\columnwidth]{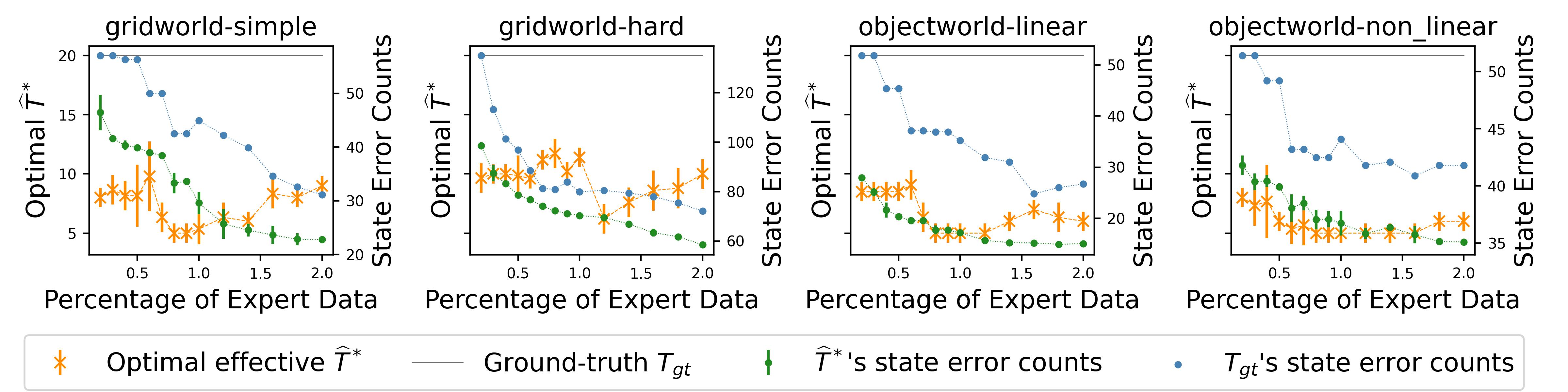}
\caption{Optimal horizons ($\optimalHorizon$) for MaxEnt-IRL at varying amount of expert data. We select the optimal $\optimalHorizon$ for each of 10 sampled task environments using the algorithm in Section~\ref{sec:cross-val}, based on the amount of expert data. Orange curves show how $\optimalHorizon$ changes with the amount of expert data, while green curves display corresponding error counts. The ground-truth $\gtHorizon = 20$ is depicted by a grey line, with corresponding error lines in blue. The trends are consistent with LP-IRL.}
\label{fig:gamma_vs_coverage_maxent}
\Description{Gamma vs Data MaxEntIRL}
\end{figure}

\begin{figure}[htp]
\centering
\includegraphics[width=1\columnwidth]{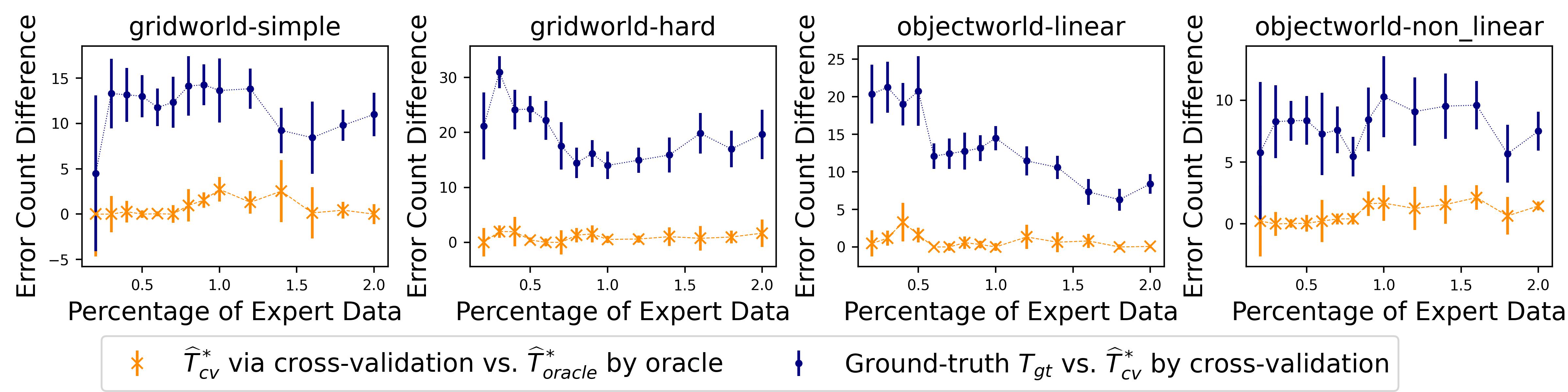}
\caption{Cross-validation results for four tasks on MaxEnt-IRL. The $x$-axis represents the amount of expert data, while the $y$-axis shows error count differences in all states for various $\optimalHorizon$s: $\optimalHorizonCv$ learned via cross-validation, and $\optimalHorizonCvAll$ chosen optimally using the oracle. Orange dots depict error count differences between the induced policies of $\optimalHorizonCv$ and $\optimalHorizonCvAll$, while blue dots represent differences between the induced polices of $\optimalHorizonCv$ and $\gtHorizon$.}
\label{fig:summary_cross_validation_maxent}
\Description{Cross Validation MaxEntIRL}
\end{figure}

The performance trends of MaxEnt-IRL align with those of LP-IRL, further supporting our theoretical results. Figure \ref{fig:maxent_envs} provides a summary of the performance. Figure \ref{fig:gamma_vs_coverage_maxent} demonstrates how the optimal horizon $\optimalHorizon$ changes with increasing amount of expert data for MaxEnt-IRL.
Figure~\ref{fig:summary_cross_validation_maxent} summarizes cross-validation results for MaxEnt-IRL.

\begin{algorithm}[H]
\caption{MaxEnt IRL with Expert Trajectory Sampling}\label{alg:maxent-irl}
\begin{algorithmic}[1]
    \Require 
    \(S\in \mathbb{R}^{N}\) (states),
      \(A\in \mathbb{R}^{m}\) (actions), discount \(\gamma\in(0,1)\), 
      \(P\in\mathbb{R}^{N\times m\times N}\) (transition tensor), horizon \(\widehat{T}\), number of trajectories \(T\), \(\Phi\in\mathbb{R}^{N\times D}\) (feature matrix), 
      epochs \(E\), and learning rate \(\eta\).
    \Ensure Reward vector \(R\in\mathbb{R}^N\).

    \State \textbf{Sample Expert Demonstrations:}
    \For{\(t=1,\dots,T\)}
       \State Sample an initial state \(s_0\)
       \State Roll out trajectory \(\tau_t = \{(s_k,a_k)\}_{k=0}^{\widehat{T}}\) using the expert policy and \(P\)
    \EndFor
    \State Let \(\mathcal{T} = \{\tau_1,\dots,\tau_T\}\)
    \State \(\tilde{\phi} \gets \texttt{find\_feature\_expectations}(\Phi, \mathcal{T}, N)\)
    
    \State \textbf{Initialize Weights:}
    \State Generate \(K\) random weight vectors \(\{\alpha_i\}_{i=1}^{K}\subset \mathbb{R}^{D}\)
    
    \State \textbf{IRL Optimization:}
    \For{\(i=1,\dots,K\)}
       \State Let \(\alpha \gets \alpha_i\)
       \For{\(j=1,\dots,E\)}
           \State \(r \gets \Phi\,\alpha\)
           \State \(\mu \gets \texttt{find\_expected\_svf}(N,\, r,\, \mathcal{T},\, \widehat{T})\)
           \State \(\nabla \gets \tilde{\phi} - \Phi^\top \mu\)
           \State \(\alpha \gets \alpha + \eta\,\nabla\)
       \EndFor
       \State \(L_i \gets \|\nabla\|_1\)
       \State \(R_i \gets \Phi\,\alpha\)
    \EndFor
    \State \(i^* \gets \arg\min_{i} L_i\)
    \State \(R \gets R_{i^*}\)
    \Return \(R\)
\end{algorithmic}
\end{algorithm}

\end{document}